%% file: main.tex
\pdfoutput=1
\documentclass{article}

\usepackage{microtype}
\usepackage{graphicx}

\usepackage{subcaption}
\usepackage{comment}
\usepackage{booktabs} 
\usepackage{hyperref}
\usepackage[nice]{nicefrac}

\usepackage[numbers]{natbib}

\usepackage{amsmath}
\usepackage{amssymb}
\usepackage{mathtools}
\usepackage{amsthm}
\usepackage{amsmath}

\usepackage{mathtools}
\usepackage{amsthm}

\usepackage{amsmath,amssymb,amsfonts}
\usepackage[margin=1in]{geometry}
\usepackage{algorithmic}
\usepackage{graphicx}
\usepackage{authblk}
\usepackage{textcomp}
\usepackage{xcolor}
\usepackage[capitalize,noabbrev]{cleveref}

\usepackage{custom}

\theoremstyle{plain}
\newtheorem{theorem}{Theorem}[section]
\newtheorem{proposition}[theorem]{Proposition}
\newtheorem{lemma}[theorem]{Lemma}

\newtheorem{corollary}[theorem]{Corollary}
\theoremstyle{definition}

\theoremstyle{remark}

\input{std_macros.tex} 
\usepackage[textsize=tiny]{todonotes}

\begin{document}

\title{Contextual Reliability: When Different Features Matter in Different Contexts}







\author[1*]{Gaurav Ghosal}
\author[2*]{Amrith Setlur}
\author[3]{Daniel S. Brown}
\author[1]{Anca D. Dragan}
\author[2]{Aditi Raghunathan}
\affil[1]{University of California, Berkeley}
\affil[2]{Carnegie Mellon University}
\affil[3]{University of Utah}
\affil[*]{denotes equal contribution}
\date{}
\vskip 0.3in

\maketitle
\begin{abstract}

Deep neural networks often fail catastrophically by relying on spurious correlations. Most prior work assumes a clear dichotomy into spurious and reliable features; however, this is often unrealistic. For example, most of the time we do not want an autonomous car to simply copy the speed of surrounding cars---we don't want our car to run a red light if a neighboring car does so. However, we cannot simply enforce invariance to next-lane speed, since it could provide valuable information about an unobservable pedestrian at a crosswalk. Thus, universally ignoring features that are sometimes (but not always) reliable can lead to non-robust performance. We formalize a new setting called \textit{contextual reliability} which accounts for the fact that the ``right'' features to use may vary depending on the context. We propose and analyze a two-stage framework called Explicit Non-spurious feature Prediction (ENP) which first identifies the relevant features to use for a given context, then trains a model to rely exclusively on these features. Our work theoretically and empirically demonstrates the advantages of ENP over existing methods and provides new benchmarks for contextual reliability.
\end{abstract}

\input{sections/macros.tex}
\input{sections/intro.tex}
\input{sections/rel_works}
\input{sections/formulation}

\input{sections/method}
\input{sections/analysis}
\input{sections/experiments}

\bibliography{main}
\bibliographystyle{unsrt}

\newpage
\appendix
\onecolumn

\input{sections/appendix.tex}

\end{document}

%% file: std_macros.tex
\newcommand\eqdef{\ensuremath{\stackrel{\rm def}{=}}} 


%% file: sections/macros.tex
\newcommand\wctestacc{\ensuremath{\text{Err}_\text{worst-case}}}

%% file: sections/intro.tex
\section{Introduction}
Despite remarkable performance on benchmarks, deep neural networks often fail catastrophically when deployed under slightly different conditions than they were trained on. Such failures are commonly attributed to the model relying on ``spurious'' features (\eg background) rather than ``non-spurious'' features that remain reliably predictive even for out of distribution inputs. Prior work has focused on learning models that rely exclusively on non-spurious features. 
\begin{figure*}
    \centering
    \includegraphics[width = 1.0\textwidth]{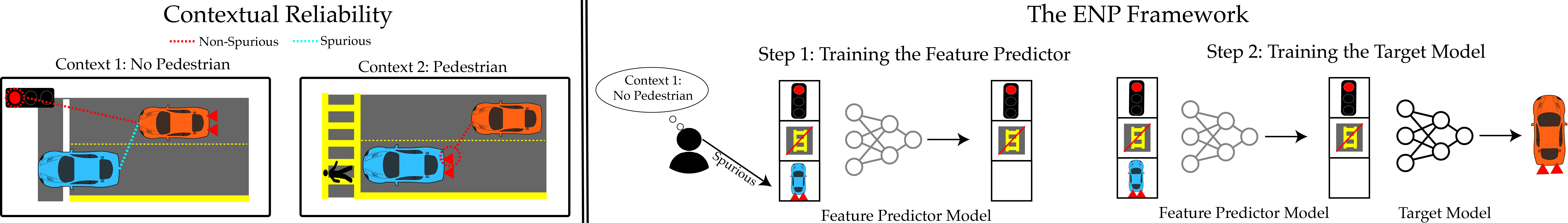}
    \caption{\textbf{Our Setting and Proposed Framework.} (\textbf{Left Panel}) Typically, it is dangerous for an autonomous car to get influenced by cars in the next lane---you do not want to run a red light if the car next to you does. Prior work provides methods to be invariant to the speed of neighboring cars. However, in the context of a pedestrian crossing, such an invariance is dangerous. The braking of the neighboring car can provide valuable information about a potential pedestrian. (\textbf{Right Panel}) We propose a two-stage framework for achieving reliable performance. In the first stage, a feature prediction model is trained to predict the set of non-spurious features from human annotations. In the second stage, we train a target model that is invariant to predicted spurious features. }

    \label{fig:my_label}
\end{figure*}

We argue that such a neat delineation of features as ``non-spurious'' and ``spurious'' is often unrealistic. As an example, consider autonomous driving. In some cases, it can be dangerous to rely on the speed of cars in the neighboring lane. A neighboring car running a red light should not cause the agent to dangerously run a red light as well, and a neighboring car slowing to turn left should not cause an agent going straight to slow down. At a crosswalk, however, the braking of a neighboring car carries evidence of an unobserved pedestrian and should be treated as a cause for stopping. Thus, we observe that neighboring-lane speed cannot be globally treated either as spurious or non-spurious, it must be used or ignored \textit{depending on the context}. Similarly, image backgrounds are often considered spurious in prior robustness research. However, we contend that there are contexts where the background is useful and should not be ignored, for example, when the foreground is occluded or ambiguous. In fact, humans often use the background to identify objects in such situations~\cite{torralba2003contextual}.

In this work, we propose and study \emph{contextual reliability}, a new setting that better captures the above nuances of real-world settings. We assume that the data comes from different latent contexts, each of which has a potentially different designation of spurious and non-spurious features. Thus, in contrast to prior settings, the optimal robust predictor might need to rely on different features in different contexts.

How to improve contextual reliability?
The predominant approach to improving robustness enforces invariance across a fixed set of spurious features \cite{muandet2013domain, NIPS2011_b571ecea, gulrajani2020search,rosenfeld2021risks} but such a set need not exist in our setting. The alternative active learning approaches do not require a predefined set of features, but actively source maximally uncertain data with the hope of breaking spurious correlations. These approaches also fail to address contextual reliability due to poor uncertainty estimates when different features are spurious across different contexts. Finally, robust optimization based approaches also fail to control for performance across multiple contexts without explicit information about the latent contexts of different training points. This suggests that we need \emph{some additional information about the latent context} in order to improve contextual reliability.

Having established the necessity of context information, we consider the problem of harnessing it effectively. 
We propose a framework called \emph{Explicit Non-spurious feature Prediction} (ENP). Ideally, we want to train a model that is invariant to the context-dependent spurious features. Rather than expecting end-to-end training to implicitly uncover contexts and respect their feature invariances, ENP employs a divide and conquer approach whereby the model first explicitly predicts the context before making a final prediction that respects the predicted context's invariances. To provide explicit supervision for the first identification step, we augment a \emph{small} fraction of the training set with explicit annotations (termed \emph{feature annotations}) on what features the optimal robust predictor should rely on. We analytically compare ENP to a variety of baselines (both with and without context information) in a simple linear setting. This allows us to precisely characterize the conditions when context information is helpful, and how different approaches to incorporating context compare. We also confirm these findings via simulations on linear models and neural networks. 

Finally, we consider a variety of semi-synthetic and real-world datasets that require contextual reliability ranging from control to image classification and motion forecasting with real-world autonomous vehicle data from the Wayo Open Motion Dataset (WOMD) \cite{ettinger2021large}. On WOMD, we make use of crowd-sourced human annotations of vehicle spuriousness provided in \cite{causalagents}. ENP offers consistent gains over baselines across \emph{all} these settings, offerings gains of around $15\%$ in control environments and $6\%$ in image classification, and $5\%$ on WOMD. Ultimately, we hope that our setting of contextual reliability and proposed ENP method serve as a setting and benchmark for addressing this important real-world challenge.

%% file: sections/rel_works.tex
\section{Related Works}
\label{relworks}

\textbf{Robustness in supervised learning.} Prior works in machine learning have investigated various distribution shift types: subpopulation shifts~\cite{hu2018does,sagawa2019distributionally}, input perturbations and adversarial shifts~\cite{goodfellow2014explaining,raghunathan2018certified}, and generalization to new domains~\cite{gulrajani2020search}. 
Our setting of contextual reliability is most closely related to subpopulation shifts and one line of work to address this is Distributionally Robust Optimization (DRO)~\cite{duchi2019distributionally,liu2014robust}. DRO optimizes the worst-subpopulation performance, which can be over-conservative and statistically inefficient~\cite{hu2018does}.
A slightly different setting is domain generalization, where the goal is to learn a predictor that extrapolates to unseen subpopulations~\cite{li2018learning}. Usually, assumptions about the relationship between domains are made in order to allow the robust predictor to be reliably identified~\cite{muandet2013domain}. Our setting of contextual reliability is similar to domain generalization in that the goal is to be optimally robust on every context, however, we do not consider the task of extrapolating to new contexts and do not rely on end-to-end training. Rather, we explicitly infer the context and enforce invariance to the spurious features in each context.

\textbf{Robustness in imitation learning.}
Imitation learning  naturally suffers from distribution shift: during training, the distribution of observed states arises from the expert policy, while in testing it arises from the learned policy \cite{ross2011reduction,tien2023causal}. As a result, prior works have contributed methods for achieving robustness to these shifts. \cite{de2019causal} propose learning the true causal graph of an expert's policy through online execution and expert queries. \cite{lyle2021resolving} examine uncertainty-based exploration for disambiguating spurious correlations in both online and imitation learning settings. Although these methods have achieved success in previously studied robustness settings, they primarily rely on careful data collection in the environment. As we demonstrate in Section 6, techniques for achieving this often fail under contextual reliability.
Other work on robust imitation learning considers Bayesian robustness to uncertainty over the objective~\cite{brown2020bayesian,javed2021policy}, in contrast to  the contextual feature reliability we study.

 \textbf{Incorporating prior knowledge.} Recent works have proposed methods for leveraging prior human knowledge to improve neural network robustness. For example, \cite{kohconceptbottleneck} introduces the concept bottleneck method for embedding interpretable human concepts within neural networks. Although the concept bottleneck shares a similar high-level approach to our paper of imposing explicit structure in neural networks, concept bottleneck models do not consider the problem of controlling when different features are utilized, as we are concerned with. Notably concept bottleneck models represent an ideal setting for our method as the presence of human-interpretable concepts makes providing the annotations we consider in this work feasible. Another approach for incorporating prior human knowledge is explicitly regularizing model saliency maps to align with human annotations \cite{Drightforrightreason}. In our work, we demonstrate that human annotation of relevant features is particularly essential in the contextual reliability setting.  However, we avoid directly regularizing saliency maps in favor of data augmentations due to the fragility of saliency methods observed in prior work \cite{fragilegradients}.

%% file: sections/formulation.tex
\section{The Setting of Contextual Reliability}
\label{sec:contextualrobustness}

\newcommand\thetaopt{\ensuremath{\theta^\star_\text{rob}}}

We first formalize relevant background and introduce the setting of contextual reliability. Next, we contrast our setting with prior distribution shift settings.

\textbf{Preliminaries.} We learn predictors that map an input ${x} \in \mc{X}$ to some target $y \in \mathcal{Y}$ where $\mathcal{Y}$ is a discrete set. The target could either be a class label as in supervised learning, or the expert action in imitation learning. We assume access to $n$ sampled training points $\{({x}^{(1)}, \y^{(1)}), \ldots ({x}^{(n)}, \y^{(n)}) \}$. 
Let $\theta \in \Theta$ parameterize the class of predictors such that $f(x; \theta) \in \Real$ and $\ell:\Real \times \mathcal{Y} \mapsto \Real$ 
is used to compute the loss $\ell(f(x; \theta), \y)$ that evaluates the prediction 
at point ${x}$, for parameter $\theta$. 

\textbf{Reliable performance.} We are interested in models that work reliably, even under shifts between the train and test distributions. We consider two settings: supervised learning and imitation learning. 
In the supervised learning setting, we are interested in training robust models that work well across \emph{all} subpopulations in the training data. For example, these partitions could each model different settings like normal driving conditions, slowdowns due to accidents etc. Typically some subpopulations (e.g. normal driving conditions) are more common than others (e.g. accident-induced slowdowns). However, at test time, it is imperative to achieve good performance in \emph{all} subpopulations including the less frequent ones. 

Formally, training inputs $\mb{z}=(\x, \y)$ are drawn from a mixture distribution over the set of $K$ latent subpopulations, \ie $\mb{z} \sim \mathsf{P} \eqdef \sum_{k \in [K]} \alpha_k \mathsf{P}_k$, where $\mathsf{P}_k$ is the distribution over the $k^\text{th}$ subpopulation. The goal is to control the worst-case performance across all subpopulations:
\begin{align}
\label{eq:metric}
 \text{Err}_\text{rob}(\theta) \eqdef \max \limits_{k \in [K]} \E \limits_{(\x, \y) \sim \mathsf{P}_k} [\ell(f(\x; \theta), \y)], 
\end{align}
where $\ell$ is some appropriate loss function.

In the imitation learning setting, distribution shifts naturally arise due to a difference between the train distribution (induced by the expert policy) and test distribution (induced by the learned policy) which often leads to poor test performance of such methods \cite{de2019causal}. Our metric of interest $\text{Err}_\text{rob}(\theta)$ in these imitation learning settings is simply the total reward obtained by the policy induced by $\theta$.

\subsection{Background: Prior Robustness Methods}
\label{sec:background}
When training deep networks, it is widely observed that simply minimizing the empirical loss on the training data leads to poor performance under subpopulation shifts ~\cite{koh2021wilds,beery2021iwildcam,zech2018variable}. Several approaches have been proposed to achieve robust performance under these shifts. 

Consider the \emph{optimal robust predictor} defined as follows. 
\begin{align}
\thetaopt \eqdef \arg\min \limits_{\theta \in \Theta} \text{Err}_\text{rob} (\theta).
\end{align}
Prior approaches consider different training methods that are aimed at retrieving $\thetaopt$. 

\textbf{Invariance-based approaches.} One popular approach to improve robustness is to enforce invariances that are displayed by the optimal robust predictor $\thetaopt$. To do so, it is convenient to think of a model as using various ``features'' $\Phi$, where each $\phi: \mathcal{X} \mapsto \Real \in \Phi$ is non-spurious with respect to the optimal robust predictor $\thetaopt$ if $f(\x; \thetaopt)$ varies as the feature $\phi(\x)$ varies. All other features are considered spurious, \ie  $f(x; \thetaopt)$ is invariant to spurious features. Some approaches assume knowledge of the spurious features and directly enforce invariance during training via appropriate augmentations~\cite{botev2022regularising} or regularizing saliency maps~\cite{ross2017right}. Other works address the case where spurious features must be inferred automatically; however, these approaches offer limited gains in practice~\citep{arjovsky2019invariant, heinze2018invariant,peters2016causal}. 

\textbf{Robust optimization approaches.} Another family of robust training methods minimize the worst-case loss across subpopulations in the training data~\cite{sagawa2019distributionally}. These methods can be viewed as minimizing the empirical counterpart of the worst-case loss $\text{Err}_\text{rob}$ described in Equation~\eqref{eq:metric}. Crucially, these methods require annotating the entire training set with the subpopulation identity.

\textbf{Targeted data collection.}  A third family of approaches for learning reliable models seeks to influence the data collection process such that the empirical risk minimizer over this new distribution is close to $\thetaopt$. This usually involves collecting more data from subpopulations that are underrepresented in the original training distribution and requires the ability to either interact with the environment ~\cite{de2019causal,lyle2021resolving}, or actively query labels for points from an unlabeled pool~\cite{activelearningtamkin}. Importantly, the success of these methods depends heavily on access to reliable uncertainty estimates~(\eg \cite{activelearningtamkin, lyle2021resolving}) that inform the collection process.

\subsection{Our Setting: Contextual Reliability}
In this section, we introduce a novel setting that better captures the nuances of reliable performance in the real world. Next, we compare it to previously studied settings. 

We consider the following small twist to the data generation process. 
Given a discrete set of contexts $\mc{C} \eqdef \{c_1, c_2, \hdots c_k \}$, we first sample a context $\cont \sim \Prob(\cont)$ (from a categorical distribution over $\mc{C}$), and then sample $\x, \y$ from a distribution $\mathsf{P}_{\cont} \eqdef p((\x, \y) \mid \cont)$. Our goal is now to achieve reliable performance in \emph{all} contexts. Formally, we are interested in the following objective: 
\begin{align}
\label{eq:contextual metric}
 \text{ConErr}_\text{rob}(\theta) \eqdef \max \limits_{\substack{k \in [K], \\ c \in \mathcal{C}}} \E \limits_{(\x, \y) \sim \mathsf{P}_{k, c}} [\ell(f(\x; \theta), y)], 
\end{align}
where $\mathsf{P}_{k, c}$ is the probability distribution over the $k^\text{th}$ subpopulation in context $c$, and $\ell$ is an appropriate loss function.

As motivation, consider an autonomous driving setting with a \textit{next-lane vehicle speed} feature and two contexts indicating the presence/absence of a pedestrian crossing. Across both contexts, agent speed and \textit{next-lane vehicle speed} are generally correlated. However, in the context of no-pedestrian crossing, the slowing of a neighboring vehicle need not imply that the agent should slow, absent of other information. For example, the neighboring vehicle may be preparing to perform a turn or responding to an obstruction  in its lane. Suddenly slowing to mimic this vehicle may be unnecessary and expose the agent to the risk of rear-end collisions. At a pedestrian crossing, however, braking of the neighboring vehicle may indicate an unobservable pedestrian entering the intersection and is evidence in itself of the need to stop. Thus, the \textit{optimal robust predictor must rely on different features in different contexts.} Consequently, both existing context-invariant approaches and their context-incorporating extensions will fail to achieve reliable performance.

\subsection{The Need to Incorporate Context Information}
Without accounting for the context, we find all prior approaches can fail. We support this finding with intuition in this section, analytical proofs in Section~\ref{sec:analysis}, and experimental observations in Section~\ref{sec:experiments}.

Invariance-based methods train models that use the same set of features in all contexts. For our autonomous vehicle example, this would result in a model that either always uses the next lane speed (with dangerous outcomes when there is no pedestrian crossing and an irrelevant neighboring car slows down) or always ignores the next lane speed (with dangerous outcomes when there is a pedestrian crossing). In the extreme case where every feature is used by the optimal robust predictor in some context, invariance-based methods reduce to standard empirical minimization which is well documented to perform poorly under distribution shifts.

Robust optimization approaches on the other hand, make no assumptions of universal invariance with respect to features. However, they also fail if we do not incorporate context. Minimizing the empirical counterpart of the objective of interest in Equation~\eqref{eq:contextual metric} requires annotations of the context of training points. In the absence of context annotations, we can only minimize the empirical counterpart of Equation~\eqref{eq:metric} which can differ wildly from Equation~\eqref{eq:contextual metric} if the contexts are imbalanced in the training data. 
Finally, our empirical investigation in Section~\ref{sec:experiments} reveals that uncertainty-based data collection methods also fail to successfully handle multiple contexts. We hypothesize that this is due to the challenging nature of forming high-quality uncertainty estimates when confronted with latent contexts.

%% file: sections/method.tex
\section{How to Incorporate Context Information?}

\label{sec:method}
In the previous section, we introduced the setting of contextual reliability, where the optimal robust predictor relies on different features in different contexts, and argued that achieving reliable performance requires access to context information. 
In this section, we explore different ways of collecting and incorporating this information into model training. We start with a natural extension of prior approaches and describe its limitations. We then present our proposed approach and demonstrate it is a viable method for addressing the limitations faced by the baseline method.

\subsection{Context Identity Annotations}
In order to incorporate context knowledge, we can annotate every training point with its corresponding context. Formally, we annotate each training point $(\x^{(i)}, \y^{(i)})$ with $\cont^{(i)}$ such that $(\x^{(i)}, \y^{(i)}) \sim \mathsf{P}_{{\cont^{(i)}}}$.

\textbf{Independent Classifier Per-Context (ICC).} With this context identity information, one natural baseline is to simply train a separate model (via empirical risk minimization) for each context. At test time, we first predict the context of the input and then use the corresponding predictor. We refer to this method as ICC (independent classifier per-context).

\textbf{Context-Based Robust Optimization (conDRO).} A more sophisticated way to leverage context annotations is via robust optimization. Robust optimization approaches already assume annotations of the subpopulation identities of the training data, where different subpopulations capture partitions of the input space across which we want to obtain good worst-case performance. Equipped with additional context identities of training points, 
we can partition the training data into $m = |\mathcal{C}| \times K$ groups and minimize the worst-case training loss across all $m$ groups. Here, we have $K$ sub-populations for each context and $|\mathcal{C}|$ contexts. This would be the empirical counterpart of our objective of interest in Equation~\eqref{eq:contextual metric}. We refer to this method as \emph{conDRO}, an extension of robust optimization with context information. 
\begin{align}
\label{eq:conDRO}
\btheta_\text{conDRO} \eqdef \arg \min \limits_{\theta \in \Theta} \max \limits_{\substack{k \in [K], \\ c \in \mathcal{C}}} \E \limits_{(\x, \y) \sim {\hat{\mathsf{P}}}_{k, c}} [\ell(f(\x; \theta), \y)], 
\end{align}
where ${\hat{\mathsf{P}}}_{k, c}$ is the empirical distribution over all training points sampled from $\mathsf{P}_{k, c}$. 

In this work, we propose a new framework for extracting and incorporating information about contexts: Explicit Non-spurious feature Prediction (ENP). We propose to use a different kind of annotation rather than the natural but naive annotation of context identities.  
\subsection{Explicit Non-Spurious Feature Prediction}
Our framework is motivated by looking more carefully at what the optimal robust predictor should do in the contextual reliability setting. Recall that under contextual reliability, the optimal robust predictor relies on different features in different contexts. Therefore, the optimal robust predictor should first infer the context, and then leverage the contextually non-spurious features, while being invariant to the spurious ones. Rather than training a model end-end in some fashion and expecting this structure to emerge due to implicit biases in the training process, we propose to collect context information and explicitly insert this structure into the predictor. 
As the context affects the optimal predictor solely by determining the set of non-spurious features, we solicit context information in the form of explicit non-spurious feature annotations (formally defined below) instead of context identities.

\textbf{Feature Annotations.} Let $\theta^\star_\text{rob}(c)$ denote the optimal robust predictor for context $c$. A feature $\phi: \mathcal{X} \mapsto \Real$ in the set of countable features $\Phi$ is non-spurious in context $c$ if the distribution of $f(x; \theta^\star_\text{rob}(c))$ is not invariant to the feature values $\phi(x)$ for inputs $x \sim \mathsf{P}_{{c}}(\cdot)$. Let $\mc{N}(c)$ denote the set of all such non-spurious features $\phi(x)$ in context $c$. We propose to annotate training point $x^{(i)}, \y^{(i)}$ with the subset of non-spurious features $\mc{N}^{(i)} = \mc{N}(c^{(i)})$ where $x^{(i)},\y^{(i)} \sim \mathsf{P}_{c^{(i)}}$. We do not require the entire training set to be annotated, only (without loss of generality) the first $n' < n$ examples.
With these annotations, we propose a two-step methodology: 

\textbf{Step one: Train a feature predictor}. Given training data $\{(x^{(1)}, \mc{N}^{(1)}), \hdots (x^{(n')}, \mc{N}^{(n')})$, we learn a predictor $g: \mathcal{X} \mapsto 2^{\Phi}$ that maps inputs to their corresponding set of non-spurious features, where $\Phi$ is the set of all features. 

\textbf{Step two: Train a target model that relies exclusively on predicted non-spurious features}. 
We train a target model that takes as input the pair of original datapoint and its feature annotations $(x,\; \mc{N})$, and  returns the prediction $f(x; \theta)$ such that $f(x; \theta)$ is invariant to the spurious features, \ie all features $\phi \in \Phi \setminus \mc{N}$. 
This model is trained on training data comprising $(x^{(i)}, y^{(i)}, \mc{N}^{(i)})$ for $i=1, \hdots n'$ and $(x^{(i)}, y^{(i)}, g(x^{(i)}))$ for $i=n'+1, \hdots n$, where $g$ is the trained feature predictor obtained from step one. In other words, we use the ground-truth feature annotations when they are provided and the \textit{predicted} features annotations on unannotated data points. 

At test time, given an input $x$, we first apply the feature prediction model $g$ to obtain non-spurious features $g(x)$. We then pass $(x, g(x))$ as input to the target model and obtain final predictions. We enforce invariance at both test and training time via augmentations that perturb the values of spurious features (either by adding noise or zeroing them out) such that they cannot be relied upon by the target model. Our core methodological contribution is this two-step process where, rather than training an end-end model, we explicitly induce the structure that different features should be used in different contexts for reliable performance.  Next, we discuss the benefits of our proposed method (ENP) over alternatives.

\subsection{Benefits of Explicit Non-Spurious Prediction}
\label{sec:benefits}
We identify three axes along which ENP outperforms alternative approaches to incorporating context information. 

\textbf{(1) Annotation cost.} Our method only requires non-spurious feature annotations on a subset of the training set and can train a downstream model on the full training set by using \emph{predicted} feature annotations. In contrast, end-to-end approaches such as conDRO require context annotations to be provided on the entire training set. 
Across a variety of semi-synthetic and real datasets, we are able to achieve good non-spurious feature prediction accuracy with just a small fraction of the training set annotated. 

\textbf{(2) Annotation feasibility.}
Often, it is easier for experts to think in terms of the reliability of features for a given input, rather than identifying a dataset-wide partition of points into contexts. It is clear to an expert that next lane speed should be used when they see a pedestrian crossing. However, simply given an input with a pedestrian crossing, it might be hard to know apriori that this corresponds to a distinct context. In recent work, ~\cite{causalagents} use a similar rationale to crowdsource the annotations of ``causal agents'' in a driving scenario (discussed in Section \ref{sec:experiments}).

\textbf{(3) Preventing overfitting.} Even if we allow context annotations on all training points, we find that methods like conDRO fail. This is an issue with the inductive bias of current training algorithms. In the limit of infinite training data, conDRO (Equation~\eqref{eq:conDRO}) should also minimize our metric of interest (Equation~\eqref{eq:contextual metric}). However, conDRO fails to do so with finite data because it can overfit and fail to learn from minority subpopulations and contexts~\citep{sagawa2019distributionally, sagawa2020investigation}. Furthermore, training such a model end-to-end suffers from a chicken and egg problem as described in~\cite{liu2021decoupling}. In order to learn how to use non-spurious features, the model must have first learned to disambiguate the context. But without an explicit signal to disambiguate contexts, the only signal to disambiguate comes from different features being non-spurious across contexts. The model needs to already know how to use non-spurious features in order to access this signal. Our ENP framework breaks this chicken-egg problem by providing explicit supervision about the set of non-spurious features.

%% file: sections/analysis.tex
\section{Analysing ENP in a Simplified Setup}
\label{sec:analysis}

In a simplified setting that distills contextual reliability, we contrast ENP that uses non-spurious feature annotations with four baselines: (i)  IRM~\cite{arjovsky2019invariant} that learns a context invariant predictor; (ii) ICC where a separate predictor is learned for each context independently;  (iii) \mbox{conDRO}~\cite{sagawa2019distributionally} that optimizes for worst context performance (IRM, conDRO and ICC use context labels); and (iv) ERM which minimizes loss on labeled examples without knowledge of contexts or feature annotations. We show why each baseline performs suboptimally (compared to ENP), either due to its over-conservative nature in learning worst-case robust/invariant predictors or due to its statistical inefficiency caused by failure to share features across contexts. ENP's two stage procedure affords benefits specifically when the non-spurious feature predictor is easier to learn compared to learning the contextually reliable predictor end to end. 
Details on the data distribution and precise objectives for  algorithms are in Appendix~\ref{sec:algorithms-det} and proofs for our theoretical results are in Appendix~\ref{sec:theory-appendix}.

\textbf{Setup.} For a binary classification problem with  $\mc{Y} \eqdef \{-1,1\}$, the inputs $x = \brck{\x_1, \x_2, \x_3}$ (where, $\x_1, \x_2, \x_3 \in \Real^{d}$) span two contexts $\mc{C} \eqdef \{c_1, c_2\}$. The feature annotations for contexts $c_1, c_2$ are denoted by masks $\rm{C}_1, \rm{C}_2  \in \{0, 1\}^{3d}$ respectively. For each context, a different set of features is non-spurious: $\{\x_1, \x_2\}$ in $\cont_1$; and $\{\x_1\}$ in $\cont_2$. Thus, $\rm{C}_{1}^{(j)} = \mb{1}(j \leq 2d)$ and $\rm{C}_{2}^{(j)} = \mb{1}(j \leq d)$ where $\rm{C}^{(j)}$ is the $j^{th}$ coordinate for annotation $\rm{C}$. For more discussion on the annotations and other details on the data distribution please refer to Appendix~\ref{sec:algorithms-det}. In this setting, we theoretically analyze estimates returned by ERM, IRM, conDRO, ICC, and ENP for a class of linear predictors $\mc{W}_{1} \eqdef \{w \in \Real^{3d} : \|w\|_2 \leq 1\}$; and empirically evaluate solutions returned when optimizing them over deep nets.

\textbf{\ours has lower asymptotic errors than conDRO, IRM and ERM.}  
In Theorem~\ref{thm:pop-loss}, for linear models, we compare the asymptotic classification errors for all algorithms ($n\rightarrow \infty$). 
We see that both conDRO and IRM yield suboptimal performance (specifically on $c_1$) because: (i)
IRM is only restricted to use the invariant feature $x_1$, which is less predictive of the label  than $x_2$ in context $c_1$ when $\gamma \ll 1$; (ii) the conDRO objective enforces its solution to have high but uniform accuracies across both $c_1$ and $c_2$, and since any component along $x_2$ would affect the losses in both contexts in opposite ways, conDRO is forced to forego components along $x_2$. On the other hand, \ours improves over both since it is allowed to use different features in $c_1$ (both $x_1, x_2$) and $c_2$ (only $x_2$).
The ERM solution relies too heavily on $x_2$ since this significantly reduces the loss in the majority context $c_1$, but leads to worse than random performance on $c_2$. This is because correlations for $x_2$ are flipped between $c_1$ and $c_2$. While it may seem that ERM suffers because $\mc{W}_1$ class does not contain a predictor that is uniformly optimal on both contexts, in the subsection that follows we show that similar failures exist even when the model class is more expressive (deep nets). On the other hand, components along $x_2$ do not effect the predictions on $c_2$ for \ours since these components are effectively masked by the annotations $\rm{C}_2$. 
The solution for ENP is also comparable to Bayes optimal solutions found by ICC as is evident from corollary~\ref{corr-ours} which follows immediately from Theorem~\ref{thm:pop-loss}.

\begin{theorem}[test accuracies on population data]
\label{thm:pop-loss}
For $\rho_1$ $\eqdef$ $\nf{\|\mu\|_2}{\sqrt{2}\sigma}$, $\rho_2$ $\eqdef$ $\nf{\|\mu\|_2}{\sqrt{2}\eta}$, and $\gamma \ll 1$,  given population access, the following test accuracies are afforded by solutions for different 
optimization objectives over $\mc{W}_1$. 
For IRM, conDRO the accuracy $\forall p_c$ is
$0.5$$\cdot \erfc(-\rho_1)$ on both $c_1, c_2$;
ERM achieves 
$0.5 \cdot \erfc(-\rho_1 \sqrt{1+\nf{1}{\gamma}})$ on $c_1$ and $0.5 \cdot \erfc(-\nf{\rho_1(\gamma-1)}{\sqrt{\gamma^2+\nf{1}{\gamma}}})$ on $c_2$ as $p_c \rightarrow 1$;
ENP achieves $\forall p_c \geq 0.5$ at least 
$0.25 \cdot \erfc(-\rho_2) \cdot \erfc(-\rho_1\sqrt{1+\nf{1}{\gamma}})$ on $c_1$ and $0.25 \cdot \erfc(-\rho_2) \cdot \erfc(-\nf{\rho_1}{\sqrt{1+\nf{1}{\gamma^3}}})$ on $c_2$; 
and ICC achieves $\forall p_c$ 
$0.5 \cdot \erfc(-\rho_1\sqrt{1+\nf{1}{\gamma}})$ on $c_1$ and $0.5 \cdot \erfc(-\rho_1\sqrt{1+\gamma})$ on $c_2$. 
Here, $\erfc(x)$$=$$\nicefrac{2}{\sqrt{\pi}} \int_{x}^{\infty} e^{-t^2} \mathrm{d}t$. 
\end{theorem}

\begin{figure*}
    \centering
    \includegraphics[width=0.85\textwidth]{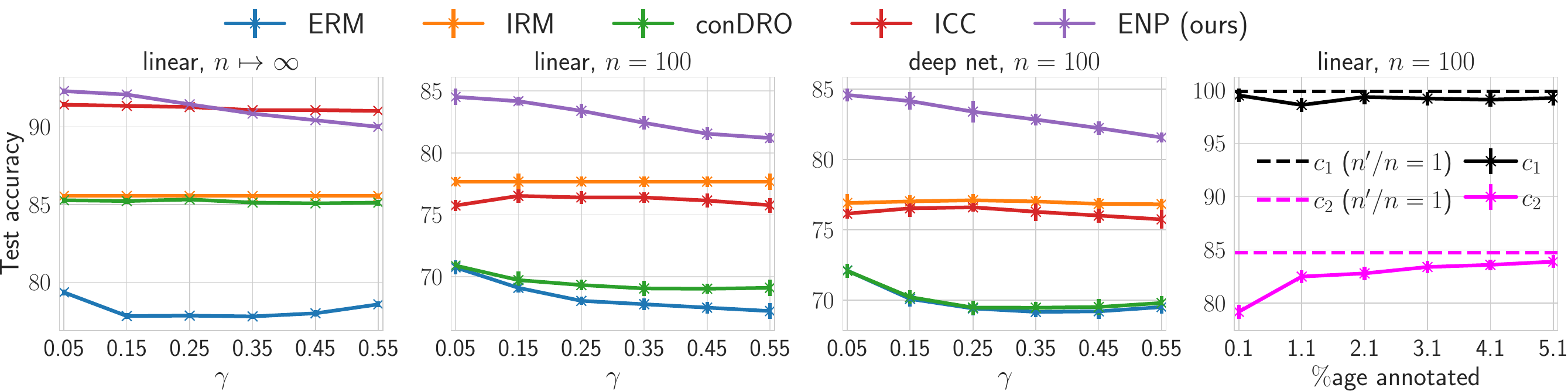}
    \caption{\textbf{Empirical evaluations of \ours and baselines (ERM, IRM, ICC, conDRO) on the simplified two-context setup in Section~\ref{sec:analysis}:} For each method, we plot the test accuracy averaged over both contexts (balanced average). In (a), (b) each method learns a  linear predictor (in $\mc{W}_1$), where the asymptotic results $(n$$\rightarrow$$\infty)$ are in (a), and the finite sample results $(n$$=$$100)$ are in (b). 
    Similarly, in (c) we compare finite sample results when the target predictor is a non-linear deep network.
    In all three plots we vary the problem parameter $\gamma$, where reducing $\gamma$ makes the feature $x_2$ more predictive of the target in $c_1$ and less predictive in $c_2$. 
    Finally, in (d) we plot the performance of \ours as we increase the fraction of non-spurious feature annotations in the training set. 
    For each value in the plot, we also report the standard deviation over five independent runs. Details on exact problem parameters for these runs are in Appendix~\ref{sec:algorithms-det}.}
    \label{fig:simulation-results}
\end{figure*}

\begin{corollary}[Almost Bayes optimality of ENP]
\label{corr-ours}
As the feature predictor in the first stage of \ours gets easier to learn $(\eta \rightarrow 0)$, the ratio of accuracies for \ours solution and Bayes optimal predictor approaches $1$ on $c_1$ and $\nf{\erfc(-\nf{\rho_1}{\sqrt{1+\nf{1}{\gamma^3}}})}{\erfc(-\rho_1)}$ on $c_2$. 

\end{corollary}

\textbf{\ours is statistically more efficient than ICC when contexts share some features.} The distribution of $\x_1$ is identical in both $c_1$ and $c_2$. But, recall that the ICC method (which also uses context annotations similar to conDRO) only relies on samples from $c_2$ to learn a classifier for $c_2$ (independent of any information from $c_1$). This can be statistically suboptimal when the distribution over contexts is skewed and consequently there are much fewer training samples drawn from $c_2$ (\ie when $p_c \rightarrow 1$). On the other hand, due to explicit annotations, ENP is aware that $x_1$ is non-spurious in both contexts. In Appendix~\ref{app:gen-error}, we formally show that the generalization error upper bound for ICC is worse than \ours by a factor of $\mc{O}({1/\sqrt{1-p_c}})$ on context $c_2$.

\subsection{Empirical Results in our Simplified Contextual Reliability Setting} 
\label{subsec:toy-empirical}

The empirical results in Figure~\ref{fig:simulation-results}(a) show asymptotic errors for the various methods and the results agree with our theoretical findings in Theorem~\ref{thm:pop-loss}. First, we see that both IRM and conDRO only learn $x_1$ and  have similar (suboptimal) asymptotic errors on $c_1$, compared to other methods (\eg ENP) that also learn feature $x_2$. 
Second, ERM has poor test accuracy since it relies too heavily on the more predictive feature $x_2$ in the majority context $c_1$, and since $x_2$ is anti-correlated on minority, ERM's accuracy on $c_2$ drops below that of random baseline. Additionally, as $\gamma$ increases we see a slight drop in the performance of \ours since the signal to noise ratio for $\x_2$ decreases in $c_1$ (where it is used by \ours) and increases in $c_2$ (where it is ignored). 
Next, in the finite sample setting, we see the poor performance of ICC in Figure~\ref{fig:simulation-results}(b), since it fails to leverage the shared feature $\x_1$ when trying to learn independent classifiers per-context. The baseline conDRO also overfits on the minority examples from $c_2$, by memorizing high dimensional noise along $\x_2, \x_3$~\cite{sagawa2020investigation}.
On the other hand, compared to ICC and conDRO, for any value of $\gamma$, using only finite data \ours achieves performance closer to that of its asymptotic value in Figure~\ref{fig:simulation-results}(a), indicating that it suffers less from finite sample estimation errors (see Theorem~\ref{thm:est-error} in Appendix~\ref{app:gen-error}). \ours improves over ICC since it can efficiently learn the feature $\x_1$ by using samples from both contexts, and it improves over conDRO since the feature annotation $\mb{C}_2$ masks out any noise along $\x_2, \x_3$ in context $c_2$, preventing memorization explicitly. 
In Figure~\ref{fig:simulation-results}(c), we plot the performance of methods when the model class is one hidden layer deep networks with $512$ ReLU activations. Here, even though the model class is expressive enough to contain predictors that are uniformly optimal over both contexts, IRM continues to fail as it enforces invariance whereas conDRO, ERM and ICC suffer more from statistical inefficiencies.
Finally, in Figure~\ref{fig:simulation-results}(d) we plot the test accuracy of \ours as the fraction of samples with feature annotations is increased. This improves the performance of the learned feature predictor which in turn improves the test performance for the target predictor, corroborating our results in Corollary~\ref{corr-ours}. 

%% file: sections/experiments.tex
\section{Experiments}
\label{sec:experiments}

In this section, we study contextual reliability in three settings spanning supervised learning for classification (Corrupted Waterbirds), imitation learning for policies (Noisy Mountain Car), and real-world vehicle trajectory prediction (Waymo Open Motion Dataset). We compare appropriate baselines in each setting to ENP and demonstrate that the theoretical benefits of ENP (Section~\ref{sec:analysis}) transfer to practice. Further experimental details can be found in Appendices~\ref{app:expdetails} and~\ref{app:womd} (for the WOMD).

\subsection{Setting One: Corrupted Waterbirds}
\textbf{Setting.} We adapt the standard Waterbirds robustness benchmark demonstrated in \cite{groupdro} to generate a data-set where the foreground bird images are blurred and randomly cropped with probability 0.05. In this setting, simply relying on the foreground bird images (as done in prior works) is suboptimal: when the foreground is corrupted, the highly correlated background provides useful information. On the other hand, when the foreground bird is unambiguous, we want to avoid relying on the background as it is not always predictive (for e.g. water birds in land background). 

\textbf{Methods.} We test the following methods on Corrupted Waterbirds: ERM, Group DRO (where groups are defined using the spurious/core features \cite{groupdro}), conDRO (group definitions are augmented with the ground-truth context), and ENP. We also compare to an oracle version of ENP where we augment according to the ground-truth (rather than predicted) context labels (GT-Aug.). We omit baselines such as Just Train Twice \cite{jtt} and Learning from Failure \cite{learnfromfailure} as they strictly underperform Group DRO. We report the worst group accuracy across contexts (corrupted and clear foreground) in Table~\ref{table:corruptwaterbirds}. 

\textbf{Results. } First, we note ERM has poor worst-group accuracy ($67\%$). Standard group DRO (which is state-of-the-art on WaterBirds) actually harms robustness in our setting with a worst-group accuracy of $60\%$. Thus, we cannot ignore the context structure for improving contextual reliability. Next, we test ENP and compare it to two methods that assume ground-truth context information on all training points: GT-Aug. and conDRO. Even with far less context information ($10\%$), ENP performs comparably to both methods ($72.8\%$). In Appendix \ref{app:ablations}, we test the performance of our feature predictor with varying feature annotation budgets.
\begin{table}
\caption{\textbf{Corrupted Waterbirds classification accuracies.} We provide the worst-group accuracy for our Corrupted Waterbirds setting, where groups are defined in terms of both the spurious attribute and context. We test methods that don't make use of context: ERM and GroupDRO (groups assigned without context information) as well as methods that use varying amounts of context information: conDRO and GT-Aug. require context information on all training points, while ENP requires only annotations on $10\%$ of the training points.  }
\vspace*{-\baselineskip} 
\label{table:corruptwaterbirds}

\begin{center}
\begin{small}
\begin{sc}
\begin{tabular}{lcccr}
\toprule
Method & Worst Case Acc.\\
\midrule
ERM &  0.67 \\
GDRO & 0.60 \\
\textbf{conDRO} & \textbf{0.73}\\
\textbf{GT-Aug.} & \textbf{0.73} \\
\textbf{ENP} & \textbf{0.728} \\
\bottomrule
\end{tabular}
\end{sc}
\end{small}
\end{center}
\vskip -0.2in
\end{table}
\subsection{Setting Two: Noisy MountainCar}

\begin{table}[t]
\caption{\textbf{Noisy MountainCar policy returns.} We show the policy evaluation returns of various imitation learning methods. We consider two standard imitation learning methods with either access to or no access to the previous action (respectively With(out) Prev. Action. We test two baselines that are successful in the universally spurious feature setting (Policy Exec Intervention and Targeted Exploration). Finally, we test conDRO (with ground-truth access to context on all points) and ENP (with ground truth access to feature annotations on 10\% of the training data).}
\label{table:mountaincarmethod}
\vspace*{-\baselineskip} 
\begin{center}
\begin{small}
\begin{sc}
\begin{tabular}{lcccr}
\toprule
Method & Test Reward \\
\midrule
Without Prev. Action & -170.4 $\pm$ 9.7\\
With Prev. Action & -194 $\pm$ 4.6\\
Policy Exec Intervention & -188.3 $\pm$ 7.2\\
Targeted Exploration & -195.2 $\pm$ 4.2\\
conDRO & -188.3 $\pm$ 4.26\\
\textbf{ENP} & \textbf{-139.5$\pm$ 13.6}\\

\bottomrule
\end{tabular}
\end{sc}
\end{small}
\end{center}
\vskip -0.2in
\end{table}

\textbf{Setting.} We study contextual reliability in imitation learning by extending the setting studied in \cite{de2019causal} where adding the previous action to the state causes the policy to underperform due to spurious correlations. In the original setting, it is optimal to always ignore the previous action. However, we hypothesize that when the state is noisy, historical actions can be useful to disambiguate it. To test this, we construct a modified version of the MountainCar environment where noise is added to the velocity in a subset of the state space. Since the optimal MountainCar policy must take different actions at a given x-position depending on which direction it is heading (up or down the slope), making selective use of the previous action is necessary to recover the heading information lost by the state noise. 

\textbf{Results for baselines.} We first test two approaches based on standard imitation learning: Without Prev. Action assumes no access to the previous action and With Prev. Action allows access to it. As shown in Table \ref{table:mountaincarmethod}, both methods perform poorly, demonstrating the insufficiency of universal invariance to the previous action. Next, we consider learning a causal graph of the optimal action through policy execution \cite{de2019causal} (Policy Exec Intervention). We find that the learned causal graph often does not contain the previous action despite it being useful on the noisy examples in our setting, resulting in poor policy performance. 

Prior work has also considered training an exploration policy to directly visit high uncertainty states and demonstrated the capabilities of this approach when there exists a universal set of reliable features. We test this method (Targeted Exploration) 
 (implementation details in Appendix \ref{app:expdetails}) and our results in Table 2 demonstrate the insufficiency of this approach and exemplify the challenges of uncertainty based active-learning in the contextual reliability setting.

 Finally we test conDRO in the imitation learning setting, defining groups using the ground-truth context, current action, and previous action. We find that it performs poorly ($-188.3$) and hypothesize that this arises from a mismatch between the conDRO objective and the evaluation metric of policy execution reward.

\textbf{Results for ENP.} We test the ENP framework in the Noisy MountainCar setting. We train a non-spurious feature prediction model by subsampling $10\%$ of the training data and providing feature annotations on these points. We use this predictor model to label all training and test states and enforce invariance through augmentations that randomly perturb the previous action (when spurious) by selecting uniformly from all possible actions. ENP's performance (-139.5) outperforms all baselines (best -170.4), showing our method is successful in appropriately making use of the previous action, while not over-relying on it. 
\subsection{Setting Three: Waymo Open Motion Dataset}
\begin{table}[t]
{\small
\caption{\textbf{Validation minADE on Waymo Open Motion Dataset.} We present the minimum average displacement error (minADE) of various MultiPath++ models, evaluated on a validation set with all agents labeled spurious by humans removed. We compare with two methods that do not use any annotation information: Standard training and Random Augmentations (where agents are randomly deleted during training). Given access to spuriousness annotations, we test a method where only human-annotated samples are augmented (Annotated Augmentations) and ENP.}
\label{table:waymo}
\vspace*{-\baselineskip} 
\begin{center}
\begin{small}
\begin{sc}
\begin{tabular}{lcccr}
\toprule
Method & Perturbed minADE  \\
\midrule
Standard & 0.817 \\
Random Augmentations& 0.815 \\
Annotated Augmentations & 0.801 \\ 
\textbf{ENP} & \textbf{0.774}\\
\bottomrule
\end{tabular}
\end{sc}
\end{small}
\end{center}
\vskip -0.1in}
\vspace{-1em}
\end{table}

\textbf{Setting.} As a preliminary evaluation of ENP on real-world data, we perform experiments using a subset of the Waymo Open Motion Dataset (WOMD) \cite{ettinger2021large} on the task of predicting the future trajectory of an autonomous vehicle. \cite{causalagents} demonstrated that many state-of-the-art models base their predictions on agents that human drivers ignore as spurious and released crowd-sourced agent spuriousness labels (termed causal agent labels) on a subset of the data. For simplicity, we treat road agents as features in this setting. We test all methods on a held-out subset of the annotated data and perturb these samples by deleting all spurious labeled agents. On a large and complex dataset such as WOMD, it can be particularly challenging to a priori specify an appropriate set of contexts, whereas pointwise annotation of spuriousness can be performed easily by human drivers. This makes it impossible to test \textit{conDRO} on WOMD and provides evidence of the enhanced annotation feasibility enjoyed by ENP (see Section \ref{sec:benefits}).

\textbf{Baselines.}
We compare ENP to two methods that do not make use of feature annotations: a standard-trained MultiPath++ model (Standard) and data augmentations which randomly delete agents during training (Random Augmentations). In addition, we examine a method that incorporates $20\%$ of the ground-truth annotated data into the training set and generates data augmentations that delete spurious-labeled agents (Annotated Augmentations). We see that Standard and Random Augmentations perform comparably ($0.817$ and $0.815$), while Annotated Augmentations result in improved performance ($0.801$). This further demonstrates the importance of incorporating spuriousness annotations for achieving reliable performance.

\textbf{ENP: Training a feature predictor.}
For the first step of ENP, we train a feature predictor by sampling $20\%$ of the annotated samples and training a model to predict the spuriousness of a given agent on this data. We find that we are able to use a much smaller architecture, relative to the full MultiPath++ model and achieve \textbf{$84.9\%$} performance on the spuriousness prediction task. The relatively small dataset size and simple model used in our method provide evidence of our hypothesis that learning the rules governing the spuriousness of agents is much easier than achieving good performance on the target task (trajectory prediction).

\textbf{ENP: Training the target model.}
Using our trained feature predictor model, we predict feature spuriousness labels on all trajectories in our trajectory prediction training set. We implement ENP analogously to Annotation Augmentations, except we are able to generate augmentations on \emph{all} data points using our feature prediction model. ENP achieves a significant improvement ($0.774$) over Annotated Augmentations ($0.801$), providing evidence of ENP's efficacy in improving model reliability with limited access to spuriousness annotations.  
\section{Conclusion}
We introduce and study a new setting of contextual reliability where it is optimal to rely on different features in different contexts. This captures several realistic settings and introduces new challenges to robust machine learning. Our theory and experiments show that methods that do not incorporate context information struggle to improve contextual reliability. Incorporating context boils down to eliciting information from an expert about the latent context. We propose and advocate for a framework that uses explicit annotations of the non-spurious features for a small fraction of the training data. 
The success of our method relies on two ingredients. The first is the ability to effectively annotate non-spurious features. 
As representation learning methods improve via large-scale pretraining, it is an interesting future direction to consider annotations in terms of higher-level learned features. 
The second ingredient is the ability to successfully learn a high-quality predictor that maps inputs to non-spurious features. We provide evidence here that this is indeed already possible for the real-world setting of motion prediction in driving. We believe an exciting line of future work is to consider even more complex context-prediction scenarios perhaps by allowing for test-time interventions with an expert.

%% file: sections/appendix.tex
\section*{Appendix Outline}

\section{Additional details for the setup in Section~\ref{sec:analysis}.}

In this section we provide details on the setup used for our analysis in section~\ref{sec:analysis}. We begin by describing the data distribution and non-spurious feature annotations for each context. Then, we provide details on the various objectives we theoretically and empirically analyze in this setup.

\label{sec:algorithms-det}

\textbf{Data distribution.} The data distribution $\mathsf{P}$ (over $\mc{X} \times \mc{Y}$) for our setup is described in \eqref{eq:data-dist} where $p_c$ is the probability of choosing the context $c_1$ (majority context when $p_c \gg 0.5$), $\eta$ is the signal to noise ratio that controls the hardness of learning the non-spurious feature predictor, and $\gamma$ $\ll$ $1$ controls the signal to noise ratio (hardness of learning) for feature $\x_2$ over $\x_1$ in $c_1$, and $\x_1$ over $\x_2$ in $c_2$. Note, that this setup distills contextual reliability in the sense that the feature $x_2$ is much more useful in predicting the label in context $c_1$ (over $c_2$), and the invariant feature $x_1$ is predictive of the label to the same degree in both contexts. 
\begin{align}
    \footnotesize
    \textrm{For}\; \bmu \in \Real^d, \; \textrm{context}\; \cont = c_1 & \; \textrm{with prob.}  \; p_c,\;\textrm{and}\; \y \sim  \mathrm{Unif}(\cbrck{-1, 1}), \nn \\
       \x_1 \mid y \;  &\sim \; \mc{N}(\bmu \y,\; \sigma^2  \I) \nn \\
     \x_2 \mid (\y, \; \cont=c_1) \;  &\sim \; \mc{N}(\bmu \y,\; \gamma \cdot  \sigma^2  \I) \nn \\
     \x_2 \mid (\y, \; \cont=c_2) \;  &\sim \; \mc{N}(-\bmu \y,\; \nf{1}{\gamma} \cdot  \sigma^2  \I) \nn \\
     \x_3 \mid \cont=c_1 \; &\sim \; \mc{N}(\bmu,\; \eta^2  \I) \nn \\
    \x_3 \mid \cont=c_2 \; &\sim \; \mc{N}(-\bmu,\; \eta^2  \I) \label{eq:data-dist}
\end{align}

\textbf{Models.} For the theoretical analysis we restrict ourselves to a linear model class. In Section~\ref{sec:analysis} we also have experimental results with deep networks.  
We use $\mc{W}_1$ to denote the class of linear predictors that are bounded in $l_2$ norm: $\mc{W}_1 \eqdef \{x \mapsto w^\top x : w \in \Real^{3d}, \; \|w\|_2 \leq 1. \}$. 
A label classifier that is used to predict the task label is evaluated using the loss function $\ell:\Real \times \mc{Y} \mapsto \Real$ evaluates classifiers $w \in \mc{W}_1$, and is a surrogate loss for the $\lzone$ error where $\lzone(w^\top x , \y) = \mathbf{1}(\sgn(w^\top x) = \y)$. 
On the other hand, a classifier that is used to predict  the context of a particular input is evaluated using the loss function $\ell':\Real \times \mc{C} \rightarrow \Real$. For the theory, both the label classifier and the non-spurious feature predictor are restricted to the linear class $\mc{W}_1$. For experiments with deep nets, we set the deep network to be a one-hidden layer ReLU network with 512 activations.

\textbf{Labels and annotations.} 
From the distribution $\mathsf{P}$ above, we are given an \textit{i.i.d.} sampled dataset $\hat{\mathsf{P}} \eqdef \{(x^{(i)}, \y^{(i)})\}_{i=1}^{n} \sim \mathsf{P}^n$. 
When clear from context, we will also use $\hat{\mathsf{P}}$ to denote to denote the empirical distribution over the sampled data.  

The context conditional distribution $(\x, \y) \mid \cont$ is denoted by ${\mathsf{P}_\cont}$. For \ours (our method) we assume access to an \textit{iid} subset of $n' \leq n$ samples for which we have the feature annotations. We have two annotations $\rm{C}_1$ and $\rm{C}_2$ for $c_1$ and $c_2$ respectively. The annotation is as follows. The $j^{th}$ co-ordinate of the feature annotation $\rm{C}^{(j)} = \one(j \leq 2d)$ if $\cont=c_1$,  and $\rm{C}^{(j)} = \one(j \leq d)$ if $\cont=c_2$. This is because, when $\gamma \ll 1$ is small, both $\x_1$ and $\x_2$ are predictive of the label in $c_1$, whereas only $\x_1$ is predictive in $c_2$ since the signal-to-noise ratio is poor for $\x_2$ in $c_2$. For the conDRO baseline, we assume access to context labels $\cont$ (but not feature annotations). For label classification, we use the exponential loss: $\ell(z, \y) = \exp(-z\cdot \y)$ where $\y \in \{-1, 1\}$ and $z \in \Real$. For context classification, we also use an exponential loss but one that now treats context $c_1$ as label $+1$ and context $c_2$ as label $-1$, \ie $\ell'(z, \cont) = \exp\paren{-z\cdot \mathbf{1}(\cont =  c_1) + z \cdot \mathbf{1}(\cont =  c_2)}$.

\textbf{Algorithms.} The goal of our analysis is to compare the performance of conDRO, ERM, and IRM with \ours by analyzing the asymptotic error for the solution found by each method, and also its statistical efficiency. Here, we write the objectives for linear predictors. For non-linear functions, the map $x \mapsto w^\top x$ is replaced with a deep neural network: $f: \mc{X}\mapsto \Real$.  

First, we begin with the ERM and IRM objectives that uses no other auxiliary information apart from the label for each example. The former minimizes average loss using all the features in the input, while the latter does the same only using the invariant feature (across contexts), which is $x_1$.  
\begin{align}
  \textrm{ERM:}&  \;\; \min_{{w} \in \mc{W}_1} \; \E_{\mathsf{P}} \ell(w^\top x,\; \y) \label{eq:erm-toy}\\
  \textrm{IRM:}& \;\; \min_{{w} \in \mc{W}_{1,\mathrm{irm}}}  \E_{\mathsf{P}} \ell(w^\top x,\; \y)  \label{eq:irm-toy},
\end{align}
where $\mc{W}_{1,\mathrm{irm}}$ is the class of norm bounded linear predictors that only use feature $x_1$ \ie, $\mc{W}_{1,\mathrm{irm}} \eqdef \{ w \in \mathcal{W}_1 : w = [w', \mathbf{0}_d, \mathbf{0}_d]\}$ ($w' \in \Real^d$ and $\mathbf{0}_d$ is a $d$-dimensional vector of $0$s). 

Next, we consider objectives that use context information (in addition to labels): (i) conDRO:  optimizes for the worst performance across contexts; (ii) ICC: learns a different classifier for each context using only samples drawn from that context. Note the ICC learns Bayes optimal predictors for each context and thus has the lowest asymptotic errors. 
\begin{align}
  \textrm{conDRO:}&   \;\; \min_{{w} \in \mc{W}_1} \max_{\cont \in \mc{C}} \; \E_{\mathsf{P}_\cont}  \ell(w^\top x,\; \y) \label{eq:conDRO-toy} \\
  \textrm{ICC:}& \;\; \min_{\begin{subarray}{r}w_1, w_2 \\ \in \mc{W}_1 \end{subarray}}  \E_{\mathsf{P}_{c_1}} \ell(w_1^\top x,\; \y) + \E_{\mathsf{P}_{c_2}} \ell(w_2^\top x,\; \y) \label{eq:icc-toy} 
  \end{align}

Finally, we describe ENP that learns two predictors: (i) non-spurious feature predictor that predicts the context (and consequently the corresponding annotation) for each test example; (ii) the label predictor  
\begin{align}
  & \textrm{\ours (target predictor)}: \;\; \min_{w \in \mc{W}_1} p_c \cdot \E_{\mathsf{P}} \ell\paren{w^\top \paren{ \rm{C}_1 \circ {\x}}},\; \y) + (1-p_c) \cdot \E_{\mathsf{P}} \ell\paren{w^\top \paren{ \rm{C}_2 \circ {\x}}},\; \y) \label{eq:ours}, \\ 
  & \textrm{\ours (feature predictor)}: \;\;  \min_{g \in \mc{W}_1} \; \E_\mathsf{P} \ell'(g^\top x, \cont), 
\end{align}
where, $\circ$ represents Hadamard product.

Note, that for conDRO we use context annotations to optimize for the worst context performance, and since this is clearly more optimal for contextual reliability (over traditional group DRO methods that do not use context information), this is the only DRO baseline we analyze. Subsequently, we shall see why even this strategy can be inefficient at learning the optimal robust predictor for each context. In the IRM objective, we restrict optimization over linear predictors that make predictions solely using $\x_1$, the only feature whose class distribution is invariant across both contexts (for each label).

\input{sections/theory_app.tex}

\section{Semi-Synthetic Experimental Details}
\label{app:expdetails}
\subsection{Corrupted Waterbirds}
\paragraph{Dataset and Architecture} As in the standard Waterbirds construction, we generated images using the CUB 2011 dataset and a subset of the Places365 dataset. However, 5\% of the CUB images were corrupted by a random crop corresponding to 30\% of the image, as well as a Gaussian Blur of radius 20. Like in the standard Waterbirds construction in \cite{groupdro}, both the test and validation datasets were generated such that the background and foreground were uncorrelated. 

In all experiments, we conducted training by fine-tuning an Imagenet-pretrained ResNet50 model (as done by \cite{groupdro}). All model weights were available to be updated during model training and a linear classification layer was appended to the model to generate the final classifications.

\paragraph{Baseline Model Training Details} For group DRO and conDRO experiments, we performed hyperparameter tuning in the intervals around the hyperparameter values used by \cite{groupdro} in their Waterbirds experiments. For the ERM and GT-Aug  experiments, we used the standard weight decay parameter of 1e-4 and tuned the best epoch using the validation dataset. 
\paragraph{ENP: Feature Predictor Model} In the Corrupted Waterbirds setting, we trained a feature predictor model to identify whether the foreground was corrupted or not and then used access to ground-truth segmentation masks to generate pixel-level feature annotations. In order to train the foreground corruption detector, we used the same architecture and hyperparameters as the standard Waterbird task (Resnet50). 
\paragraph{ENP: Target Model Invariance} Given the pixel-level spuriousness labels obtained from our feature predictor model, we generated enforced invariance to the spurious-labelled pixels by generating augmentations that added Gaussian noise to them but had the same label as the original sample. At test-time, we further enforced invariance by generating predicted pixel-level spuriousness labels, generating a fixed number of augmentations per sample, and using the averaged logits to compute the final classification. 
\paragraph{Test Set Construction and Metric} We used the standard training/test/validation designations from the WOMD. In addition, we assume ground-truth segmentations of foreground and background on test and validation datasets in order to generate augmentations (for both GT-Augs. and ENP). We report the worst-context-group accuracy on the test set (using ground-truth contexts) except we exclude the two groups in which the spurious correlation breaks and the foreground is corrupted (since under this setting, it is impossible to identify the correct label and these groups have very low accuracy. Thus, our metric is the empirical counterpart of :
\begin{align}
\min_{\cont \in \mc{C'} k \in \mc{K}} \; \E_{\mc{P}_{\cont,k}}  \one\{\w(\mb{\x})=\y\}
\end{align}
where $\mc{C'}$ denotes all context-groups except the corrupted-correlation breaking ones.
\subsection{Noisy MountainCar}
\paragraph{Environment, Data, and Architecture} We used an expert MountainCar policy in order to generate a demonstration dataset consisting of 100 demonstrations. During post-processing, we applied heavy Gaussian noise (\texttt{stdev} = 0.07) to the velocity component of the state and clipped the resulting values within the permissible range for the feature value. We used a three-layer policy network with a hidden layer of size 50 (as implemented by \cite{de2019causal}). For training the causal-graph parameterized policy we used a larger 4-layer network - with the same hidden layer size of 50 neurons. Our implementation of this environment followed the open-sourced code released by \cite{de2019causal} found at https://github.com/pimdh/causal-confusion/.

\paragraph{Baseline details} We trained all models for 80 epochs and performed model selection by performing online policy evaluation. For our standard imitation learning baselines (With Prev. Action) and (Without Prev. Action), as well as the Policy Exec. Intervention, we tuned hyperparameters on an interval around the final values used by \cite{de2019causal}. We implemented the targeted exploration \cite{lyle2021resolving} baseline by training an ensemble of imitation learning policies on the imitation learning dataset and then training an exploration policy using  proximal policy optimization (PPO) ensemble uncertainty as the reward function. We ran this policy online, collected states visited, and added them (with their corresponding expert action into the imitation learning dataset. Finally, for our conDRO method, we devised groups according to the current action (core feature), previous action (spurious feature), and group. As a result, we had a total of $3 \times 3 \times 2 = 18$ groups and we tuned both weight decay and learning rate in the range $\{1e-5,1e-4,1e-3,1e-2,1e-1\}$. For all baselines except policy execution interventions, we used the same model architecture as standard imitation learning 

\paragraph{Evaluation and Metric} All imitation learning policies were evaluated with online execution in the modified MountainCar environment (with states noised on the subset of the state space) and with access to the previous action feature. We reported the average reward attained by the imitation learning agent over 10 independent runs (i.e. independent imitation learning datasets and trained models). The reward function for MountainCar is sparse (as reward is only attained once the goal is reached and negative reward until that time) and the minimum value of $-200$ is attained when the goal is not reached.

\paragraph{ENP: Training a feature predictor} We used the same architecture as the imitation learning model and trained on the 3-target classification problem of predicting the subset of reliable features (since our augmented state vector contained 3 features. We trained this model with feature annotations on $10\%$ of our training data and found this was sufficient for $100\%$ validation accuracy.

\paragraph{ENP: Training the target model} We trained our target model using the standard imitation learning loss with data augmentations to enforce invariance to the spurious features. Since the only potentially spurious feature was the previous action, we generated augmentations which (when the feature was labelled as spurious) randomly perturbed the previous action by selecting uniformly from all actions. We generated these augmentations at training and test time (using the predicted feature annotations from our model). 

 \section{Extended Discussion and Implementation of WOMD}
 \label{app:womd}
  \subsection{Dataset and Architecture Details} 
  \paragraph{Dataset} The Waymo Open Motion Dataset (WOMD) consists of vehicle trajectory data collected on real roads as an autonomous vehicle navigates diverse traffic scenarios (intersections, traffic lights, etc.) alongside a variety of other road users (i.e., other cars, pedestrians, and cyclists). In this setting, the number of contexts is unclear and each input has a varying number of spurious/non-spurious features. As noted by \cite{ettinger2021large}, $46\%$ of driving scenes in this dataset have over 32 nearby agents, $57\%$ of the scenes have a pedestrian (with $20\%$ having more than 4), and $16\%$ of all scenes have at least 1 cyclist. As such, we believe that WOMD is representative of the real-world autonomous driving settings where there could be a diverse range of interactions between multiple road agents. The task in this dataset is to predict the autonomous vehicle (AV) trajectory given the historical trajectory of both the autonomous vehicle and other agents. 
  
  \paragraph{Base Target Model} Our experiments are conducted on the MultiPath++ \cite{multipathpp} trajectory prediction model (using the implementation at https://github.com/stepankonev/waymo-motion-prediction-challenge-2022-multipath-plus-plus) which is currently a state-of-the-art model for vehicle motion prediction tasks. The MultiPath++ model consists of LSTM trajectory encoders and fully connected road-graph polyline encoders followed by multi-context gating layers to model interactions between agents and fuse the road and agent information. Finally, a multi-context gating-based decoding layer generates a set of candidate predicted trajectories (see [3] for more information). In total, this model consists of 21 million parameters. Recently, \cite{causalagents} released a subset of WOMD with labels for whether nearby agents presented spurious or robust information with respect to the prediction of the AV trajectory. These labels were collected through a large-scale human annotation process where annotators were shown driving scenes from the perspective of the autonomous vehicle and were asked to select non-spurious agents through a web-based interface \cite{causalagents}. 
  \subsection{Training Details}
  \begin{table}
    \caption{\textbf{MultiPath++ Training Hyperparameters} We show the set of hyperparameters used in training all MultiPath++ models in our WOMD experiments.}
    \label{table:hparams}
    \vskip 0.15in
    \begin{center}
    \begin{small}
    \begin{sc}
    \begin{tabular}{lccccr}
    \toprule
    Parameter & Value \\
    \midrule
    Batch Size & 42 \\
    Learning Rate & 1e-4\\
    Gradient Norm Clipping & 0.4\\
    Mask History Percentage & 0.15\\
    Total Training Epochs & 120\\
    Learning Rate Scheduler-Type & Reduce on Plateau \\
    Learning Rate Scheduler-Factor & 0.5\\
    Learning Rate Scheduler-Patience & 20\\
    \bottomrule
    \end{tabular}
    \end{sc}
    \end{small}
    \end{center}
    \vskip -0.1in
    \end{table}
  \paragraph{Data Preprocessing} We pre-processed data according to the reference implementation of MultiPath++ with some minor modifications. Due to computational constraints, we selected a random sample of the full WOMD dataset by downloading 100 shards from the Google Cloud Store. As the human-labelers for agent spuriousness were presented with the autonomous vehicle's (AV) point of view when labeling, we only trained our model to predict the trajectory of the AV. During data preprocessing, all agent trajectory data (positions, orientations, and velocities) was transformed into the autonomous vehicle's reference frame before being fed into the MultiPath++ model. In many driving scenarios, there were agents labeled as invalid, for example, due to not being in the autonomous vehicle's field of view. In these cases, we zeroed out all agent data corresponding, as well as setting the \emph{valid} feature (part of the canonical feature representation to 0). 
  
  \paragraph{Standard Training Details} We used all standard hyperparameters released in the reference WOMD implementation (found in the file final\_RoP\_Cov\_Single.yaml and shown in the Table \ref{table:hparams}. We also tested larger learning rate parameters in the set $\{0.01,0.001,0.00001\}$ and did not find improvements with these parameters. We leave more intensive hyparameter tuning experiments for future work. 

  \paragraph{Test Set Construction and Metric} We sourced our test set as a subset of the annotated driving scenarios contained in WOMD. As specified in \cite{causalagents}, we used the spuriousness labels in order to delete all spurious labels from test set (by setting the valid feature of these agents to 0) and zeroing out the associated data. Importantly, we note that our test set was a \textit{subset} of the annotated data: we reserved $20\%$ of this data to use during training models that used agent spuriousness annotations. We computed the minimum average displacement error (minADE) as our final metric as shown in Equation \ref{eq:minade}:
  \begin{equation}
  \label{eq:minade}
    \min_{i \in [1,6]} \frac{1}{T} \sum_{j=1}^{T} ||t^{\text{gt}}_{j} - t^{\text{pred}, i}_{j}||
 \end{equation}
\paragraph{Data Augmentation Details} We adapt our data-augmentations strategy from \cite{causalagents}. As introduced by that work, driving scenarios with associated spuriousness annotations were generated by randomly deleting (i.e. setting the valid feature to 0) all spurious-labelled agents with $10\%$ probability. In our implementation of Annotation Augmentations, we followed this procedure exactly: $20\%$ of the annotated data was added to the WOMD training dataset and all these added points were augmented according to the spuriousness labels. In our annotation-free baseline, Random Augmentations, we simply performed random deletion across \emph{all agents}. In section \ref{app:enpdetails}, we describe how augmentations were performed in ENP. 
  
  \subsection{ENP Details}
  
  \label{app:enpdetails}
  \paragraph{Feature Annotation Model} We designed a lightweight feature annotation model based off of the Multipath++ architecture. Due to the variable number of agents in the scene, we opted to train an agent-conditioned model which took the road graph and other global information as input, as well as the trajectory for a given agent (in autonomous vehicle coordinates) and predicted spuriousness of the provided information. Therefore, we included all road graph embedding modules from the Multipath++ model and a single LSTM encoder for accepting the autonomous vehicle trajectory. All representations from these modules were concatenated and fed through a fully connected network in order to output the predicted spuriousness attribute. During preprocessing for our feature-prediction training set, we subsampled the number of invalid labeled agents in order to ensure dataset balance. 
  \paragraph{Target Model Training} With our feature predictor, we went through all trajectories in our training set and labeled each agent as spurious or non-spurious using our model. During MultiPath++ training, we adopted an identical approach to the Annotated Augmentations except now all trajectories were augmented in accordance with the feature predictor's labels. Although the ENP framework also involves test-time augmentations, these were not applicable in the WOMD setting because all spurious agents were already removed from the dataset (also identical to the implementation of Annotated Augmentations). 
 \section{Ablations}
 \label{app:ablations}

In this section, we conduct ablations on the number of explicit non-spurious feature annotated samples. In Table 5, we show the feature predictor accuracy given different percentages of feature annotations on Waterbirds and find that it is very high even with a very small percentage of annotated samples. We see a similar effect with the WOMD accuracy though the accuracy begins to decay quickly on smaller subsets of annotated samples (Table 6).
\begin{table}
\caption{\textbf{Corrupted Waterbirds Ablation on Annotated Samples.} We show the effect of different training set sizes on the accuracy of the feature predictor on Corrupted Waterbirds.}
\label{table:corruptwaterbirds-2}
\vskip 0.15in
\begin{center}
\begin{small}
\begin{sc}
\begin{tabular}{lccccr}
\toprule
$\%$ Annotated Training & 0.5 & 1 & 2 & 5 & 10 \\
\midrule
Feature Predictor Acc. & $90\%$ & $95.3\%$ & $97.5\%$ & $99.7\%$ & $99.9\%$ \\
\bottomrule
\end{tabular}
\end{sc}
\end{small}
\end{center}
\vskip -0.1in
\end{table}

\begin{table}
\caption{\textbf{WOMD Ablation on Annotated Samples.} We show the effect of different training set sizes on the accuracy of the feature predictor on the WOMD dataset.}
\label{table:corruptwaterbirds-3}
\vskip 0.15in
\begin{center}
\begin{small}
\begin{sc}
\begin{tabular}{lccccr}
\toprule
$\%$ Annotated Training & 0.1 & 1 & 5 & 20 & 50 \\
\midrule
Feature Predictor Acc. & $61\%$ & $77\%$ & $83\%$ & $84.9\%$ & $85\%$ \\
\bottomrule
\end{tabular}
\end{sc}
\end{small}
\end{center}
\vskip -0.1in
\end{table}

%% file: sections/theory_app.tex
\section{Omitted proofs and formal statements for the analysis in Section~\ref{sec:analysis}}
\label{sec:theory-appendix}

In this section, we provide proofs for our theorem statements in Section~\ref{sec:analysis} of the main paper. We also provide formal discussion on the generalization results for ENP and ICC.

\subsection{Proof for Theorem~\ref{thm:pop-loss}}

In this subsection, we prove claims regarding the asymptotic errors attained by solving population versions of the objectives in Section~\ref{sec:algorithms-det}, when the model class is linear $(\mathcal{W}_1)$. We look at each objective separately, but before that we introduce the following two  lemmas on optimal linear target predictors for each context, and accuracies on each context.

\begin{lemma}[optimal linear predictors for $c_1, c_2$]
The linear predictor in $\mc{W}_1$ with the least $\lzone$ error for context $c_1$ is $\nicefrac{1}{(\|\mu\|_2\sqrt{1+\gamma^2})} \cdot \brck{\mu\gamma, \mu, \mathbf{0}_d}$, and for context $c_2$ is $\nicefrac{1}{(\|\mu\|_2\sqrt{1+\gamma^2})} \cdot \brck{\mu, -\mu \gamma, \mathbf{0}_d}$. Here $\lzone$ is the 0-1 loss: $\lzone(z, \y) = \one(\sgn(z) = \y)$.    
\label{lem:opt-linear}
\end{lemma}
\begin{proof}
 For Gaussian data with the same covariance matrices for class conditionals $\mathsf{P}(x\mid\y=1)$ and $\mathsf{P}(x\mid\y=-1)$, the Bayes decision rule is given by the Fisher's linear discriminant direction (Chapter 4; \cite{bishop2006pattern}): 
    \begin{align*}
        h(x) = \begin{cases}
                1, & \text{if } h^\top x > 0 \\
                0, & \text{otherwise}
                \end{cases}
    \end{align*}
    where $h = 2\cdot \nicefrac{1}{\sigma^2}\brck{\bmu,  \nf{\bmu}{\gamma}, \mathbf{0}_d}$ for context $c_1$, and $h = 2\cdot \nicefrac{1}{\sigma^2}\brck{\bmu, -{\gamma} \bmu, \mathbf{0}_d}$ for context $c_2$ (using the covariance matrices from the data distribution for each context). Here, $\mathbf{0}_d$ is a $d-$dimensional vector of $0$s. Since, the direction of $h$ solely determines the $\lzone$ error of the predictor, the optimal linear predictors in $\mc{W}_1$ are obtained by dividing them both by their corresponding norms. 
\end{proof}

\begin{lemma}[per-context accuracy]
    The accuracy of predictor $w = \brck{w_1, w_2, \mathbf{0}_d} \in \mc{W}_1$ on context $c_1$ is $0.5\cdot\erfc\paren{-\frac{(w_1+w_2)^\top\bmu}{\sigma\sqrt{2(\|w_1\|_2^2 + \gamma \|w_2\|_2^2)}}}$ and on context $c_2$ is $0.5\cdot\erfc\paren{-\frac{(w_1-w_2)^\top\bmu}{\sigma\sqrt{2(\|w_1\|_2^2 + \nicefrac{1}{\gamma} \|w_2\|_2^2)}}}$, where $\erfc(x) = \nicefrac{2}{\sqrt{\pi}} \int_{x}^{\infty} e^{-t^2} \mathrm{d}t$. 
    \label{lem:per-context}
\end{lemma}
\begin{proof}
    Let $\mathsf{P}_{c_1}$ be the probability distribution for context $c_1$, and $\mathsf{P}_{c_2}$ be the distribution for $c_2$.  Let $z_1$ and $z_2$ be random variables distributed as $\mc{N}(\mb{0}_d, \sigma^2 \I)$. Then the accuracy on context $c_1$ is,
    \begin{align*}
        \mathsf{P}_{c_1}(\sgn \left( w^\top x \right) =y) &= \mathsf{P}_{c_1}(\sgn \left( w^\top x \right) y > 0) \\
        &= \mathsf{P}_{c_1}(yw_1^\top\bmu + yw_2^\top\bmu + yw_1^\top z_1 + yw_2^\top z_2 > 0) \\
        &= \mathsf{P}(\tilde{z} > 0) \\ 
        &= \mathsf{P}(\nicefrac{\tilde{z} - \tilde{\bmu}}{\tilde{\sigma}} > -\nicefrac{\tilde{\bmu}}{\tilde{\sigma}}) \\
        &= \mathsf{P}(z > -\nicefrac{\tilde{\bmu}}{\tilde{\sigma}}) \\
        &= 0.5 \cdot \erfc(- \nicefrac{\tilde{\bmu}}{\sqrt{2}\tilde{\sigma}}),
    \end{align*}
    where $z$ is distributed as standard Gaussian, and $\tilde{z}$ is a Gaussian random variable with mean $\tilde{\bmu} \eqdef \bmu^\top(w_1 + w_2)$ and variance $\tilde{\sigma}^2 \eqdef (\|w_1\|^2 + \gamma \cdot \|w_2\|_2^2)\sigma^2$. The last equality uses the definition of the $\erfc(\cdot)$ function. 
    The calculation for accuracy on $c_2$ remains the same except now  $\tilde{\bmu} \eqdef \bmu^\top(w_1 - w_2)$, and $\tilde{\sigma}^2 \eqdef (\|w_1\|^2 + \nf{1}{\gamma} \cdot \|w_2\|_2^2)\sigma^2$.  
\end{proof}

\begin{lemma}[solutions lie in a low dimensional subspace] For ERM  conDRO, and ENP,  their corresponding solutions would belong to the set $\mc{U} \eqdef \{\lambda_1 \cdot [\bmu, \mb{0}_d, \mb{0}_d] + \lambda_2 \cdot [\mb{0}_d, \bmu, \mb{0}_d] : \lambda_1^2 + \lambda_2^2 = 1.\}$.
\label{lem:subspace}
\end{lemma}
\begin{proof}

    First, we will show that the component along $x_3$ will be $0$. Let's say the component along $x_i$ is $w_i$. Then, for any context, the conditional variance  $\rm{V}[\y(w^\top x) \mid \y]$, denoted as $\sigma_0^2$ is $\sigma_0^2 \eqdef w_1^\top \mathrm{V}[x_1 \mid \y] w_1 + w_2^\top \mathrm{V}[x_2 \mid \y] w_2 + w_3^\top \mathrm{V}[x_3 \mid \y] w_3$, and the mean is $\bmu_0 \eqdef  w_1^\top \bmu + \one(\cont=c_1) w_2^\top \bmu - \one(\cont=c_2) w_2^\top \bmu$. Here, $\mathrm{V}[x] \in \Real^{d \times d}$ is  a positive semidefinite covariance matrix. For any context $c_1$ or $c_2$, the per-context accuracy improves as $\sigma_0$ decreases (as per Lemma~\ref{lem:per-context}) without changing $\bmu_0$. This is true when $\|w_3\|_2$ decreases monotonically. Since the loss is classification calibrated, the loss also decreases monotonically as $\|w_3\|_2$ decreases. Hence, the optimal solution would necessarily have $w_3 = \mb{0}_d$.

    Next, we consider the component along $x_1$ and assume $x_1 = \alpha_1 \cdot \bmu + \alpha_2 \cdot v$. Assume that for the solutions of ERM, ENP and conDRO: $\omega_1,\omega_2 \neq 0$ and $v^\top \bmu = 0$. The component $\alpha_2 \cdot v$ will contribute to $\sigma_0^2$ with the additive term $\alpha_2^2 \sigma^2 \|v\|_2^2$, without having any effect on $\bmu_0$. This means that we can improve the accuracy for both contexts (reduce loss $\ell$) further by reducing $\alpha_2$. This contradicts the assumption that $\alpha_2 \neq 0$ for the solutions of ERM, conDRO and ENP. Thus, for all the objectives the solution would not have any component in the null space of $\bmu$ along $x_1$. Similar argument can be used to prove that the component along the null space $\bmu$ would be zero for $x_2$ as well. 

    Combining the above two arguments on the component along $x_3$ and components along null space of $\bmu$ for $x_1$, $x_2$ we can conclude that the solutions for ERM, conDRO and ENP would necessarily lie in the two rank subspace $\mc{U}$. 
\end{proof}

Now, we are ready to start the proof of Theorem~\ref{thm:pop-loss}, and for the benefit of the reader we shall first restate the theorem statement. 

\begin{theorem}[test accuracies on population data (restated)]
\label{thm:pop-loss-restated}
For $\rho_1$ $\eqdef$ $\nf{\|\mu\|_2}{\sqrt{2}\sigma}$, $\rho_2$ $\eqdef$ $\nf{\|\mu\|_2}{\sqrt{2}\eta}$, and $\gamma \ll 1$, given population access, the following test accuracies are afforded by solutions for different 
optimization objectives over $\mc{W}_1$. 
For IRM, conDRO the accuracy $\forall p_c$ is
$0.5$$\cdot \erfc(-\rho_1)$ on both $c_1, c_2$;
ERM achieves 
$0.5 \cdot \erfc(-\rho_1 \sqrt{1+\nf{1}{\gamma}})$ on $c_1$ and $0.5 \cdot \erfc(-\nf{\rho_1(\gamma-1)}{\sqrt{\gamma^2+\nf{1}{\gamma}}})$ on $c_2$ as $p_c \rightarrow 1$;
ENP achieves $\forall p_c \geq 0.5$ at least 
$0.25 \cdot \erfc(-\rho_2) \cdot \erfc(-\rho_1\sqrt{1+\nf{1}{\gamma}})$ on $c_1$ and $0.25 \cdot \erfc(-\rho_2) \cdot \erfc(-\nf{\rho_1}{\sqrt{1+\nf{1}{\gamma^3}}})$ on $c_2$; 
and ICC achieves $\forall p_c$ 
$0.5 \cdot \erfc(-\rho_1\sqrt{1+\nf{1}{\gamma}})$ on $c_1$ and $0.5 \cdot \erfc(-\rho_1\sqrt{1+\gamma})$ on $c_2$. Note that $\erfc(x)\rightarrow 2$ as $x\rightarrow -\infty$ since $\erfc(x) = \nicefrac{2}{\sqrt{\pi}} \int_{x}^{\infty} e^{-t^2} \mathrm{d}t$. 
\end{theorem}

\begin{proof} We start with the easier cases of IRM and ICC where we directly use results from the above two lemmas. Then we shall look at conDRO and ERM where we need to deal with mixture of per-context losses. Finally, we look at \ours, where we need to analyze both feature and target predictors.

\emph{IRM.~~}
Recall that the $\mc{W}_\mathrm{irm,1}$ class only consists of unit norm bounded predictors along attribute $x_1$. Since the exponential loss is a surrogate~\cite{duchi2018multiclass}, the predictor minimizing the exponential loss $\ell$ is also the one with the highest 0-1 accuracy. Thus, we can use similar arguments as in Lemma~\ref{lem:opt-linear} to conclude that the optimal predictor is $\nicefrac{\bmu}{\|\bmu\|_2}$, and from arguments similar to the ones in Lemma~\ref{lem:per-context} we can conclude that the target accuracy $0.5 \cdot \erfc(-\nicefrac{\|\bmu\|_2}{\sqrt{2}\sigma}) = 0.5 \cdot \erfc(-\rho_1)$.

\emph{ICC.~~}
Since the exponential loss $\ell$ is classification calibrated, the minimizer of this loss on $c_1$ and $c_2$ individually also has the least $\lzone$ error in $\mc{W}_1$, which is exactly the predictor defined in Lemma~\ref{lem:opt-linear}. Directly applying Lemma~\ref{lem:per-context} on this predictor, with $w_1 = \nf{\bmu\gamma}{\|\bmu\|_2\sqrt{1+\gamma^2}}$ and $w_2 = \nf{\bmu}{\|\bmu\|_2\sqrt{1+\gamma^2}}$ we conclude that test accuracy for ICC predictor on $c_1$ is $0.5\cdot\erfc\paren{-\nicefrac{(w_1+w_2)^\top\bmu}{\sigma\sqrt{2(\|w_1\|_2^2 + \gamma \|w_2\|_2^2)}}}$. 
Similarly, with $w_1 = \nf{\bmu}{\|\bmu\|_2\sqrt{1+\gamma^2}}$ and $w_2 = \nf{-\bmu\gamma}{\|\bmu\|_2\sqrt{1+\gamma^2}}= 0.5 \cdot \erfc(-\rho_1 \sqrt{\nf{1}{\gamma}+1})$, the test accuracy of ICC on $c_2$ is $0.5\cdot\erfc\paren{-\nicefrac{(w_1-w_2)^\top\bmu}{\sigma\sqrt{2(\|w_1\|_2^2 + \nicefrac{1}{\gamma} \|w_2\|_2^2)}}} = 0.5 \cdot \erfc(-\rho_1 \sqrt{1+\gamma})$.

\emph{conDRO.~~} From Lemma~\ref{lem:subspace} we know that the solution for conDRO is of the form $\lambda_1^\star \cdot v_1 + \lambda_2^\star \cdot v_2$ where $v_1 \eqdef [\bmu, \mb{0}_d, \mb{0}_d]$ and $v_2 \eqdef [\mb{0}_d, \bmu, \mb{0}_d]$. Recall that $\rho_1 \eqdef \nf{\|\bmu\|_2}{\sqrt{2}\sigma}$. Since the exponential loss is classification calibrated and $(\lambda_1^\star)^2 + (\lambda_2^\star)^2 = 1$, we can say that:
\begin{align*}
    \lambda_1^\star \;\; &\in  \;\;  \arg\inf\limits_{\substack{\lambda_1 \in [-1, 1], \\ \lambda_2^2 = 1-\lambda_1^2}} \;  \max\paren{\E_{\mathsf{P}_{c_1}}\ell(\lambda_1 \bmu^\top x_1  + \lambda_2\bmu^\top x_2, \y), \E_{\mathsf{P}_{c_2}} \ell(\lambda_1 \bmu^\top x_1  + \lambda_2\bmu^\top x_2, \y)} \\
   \;\; &= \;\; \arg\sup\limits_{\substack{\lambda_1 \in [-1, 1], \\ \lambda_2^2 = 1-\lambda_1^2}} \; \min\paren{\erfc\paren{-\rho_1 \cdot \frac{\lambda_1 + \lambda_2}{\sqrt{\lambda_1^2 + \gamma \lambda_2^2}}}, \erfc\paren{-\rho_1 \cdot \frac{\lambda_1 - \lambda_2}{\sqrt{\lambda_1^2 + \nf{\lambda_2^2}{\gamma}}}}} \\
   \;\; &= \;\; \arg\sup\limits_{\lambda_1 \in [-1, 1]} \; \min\paren{\erfc\paren{{-\rho_1 \cdot \frac{\lambda_1 \pm \sqrt{1-\lambda_1^2}}{\sqrt{\lambda_1^2 + \gamma{(1-\lambda_1^2)}}}}}, \erfc\paren{{-\rho_1 \cdot \frac{\lambda_1 \mp \sqrt{1-\lambda_1^2}}{\sqrt{\lambda_1^2 + \nf{(1-\lambda_1^2)}{\gamma}}}}}}
\end{align*}

Note that to minimize $\erfc(\cdot)$ terms we need to increase the value of $c$ when the terms are of the form $\erfc(-\rho_1 \cdot c)$. Thus it is clear that $\lambda_1^\star > 0$. Further, since $\gamma \ll 1$, we also know that $\lambda_1^2 +\gamma (1-\lambda_1^2) <  \lambda_1^2 + (1/\gamma) \cdot (1-\lambda_1^2)$. Thus, if assume that $\lambda_2^\star \geq 0$, then the optimal value is $\lambda_1^\star = 1$ and $\lambda_2^\star = 0$. On the other hand, if  we assume that $\lambda_2^\star < 0$, then the minimum of the $\erfc$ terms is clearly lower than $\erfc(-\rho_1)$, which would be the value of the above objective at $\lambda_1^\star=1$. Therefore, we conclude that $\lambda_1^\star = 1, \lambda_2^\star = 0$, which yields the following solution for conDRO: $\brck{\nf{\bmu}{\|\bmu\|_2}, \mathbf{0}_d, \mathbf{0}_d}$. From Lemma~\ref{lem:per-context} we know that on both  contexts this solution has accuracy $0.5 \cdot \erfc(-\rho_1)$ which also matches the performance of IRM.    

\emph{ERM.~~} Once again because of classification calibrated losses, and Lemma~\ref{lem:subspace}, similar to conDRO, we can re-write the ERM problem as the following optimization objective:
\begin{align*}
    \inf\limits_{\lambda_1 \in [-1, 1], \lambda_2^2 = 1-\lambda_1^2} \;  p_c \cdot \E_{\mathsf{P}_{c_1}}\ell(\lambda_1 \bmu^\top x_1  + \lambda_2\bmu^\top x_2, \y) \; + \; (1-p_c) \cdot \E_{\mathsf{P}_{c_2}} \ell(\lambda_1 \bmu^\top x_1  + \lambda_2\bmu^\top x_2, \y) 
\end{align*}
Since, for every $p_c$ we can construct a Cauchy sequence of $\lambda^{(1)}, \lambda^{(2)}, \ldots, \lambda^{(n)}$ that converges uniformly to the $\arg\inf$ at the given value of $p_c$, we can apply Moore-Osgood theorem. The following interchanges limits and finds the solution for ERM at $p_c \rightarrow 1$.
\begin{align*}
\lambda_1^\star \;\; &\in  \;\;  \arg\inf\limits_{\substack{\lambda_1 \in [-1, 1], \\ \lambda_2^2 = 1-\lambda_1^2}} \;  \E_{\mathsf{P}_{c_1}}\ell(\lambda_1 \bmu^\top x_1  + \lambda_2\bmu^\top x_2, \y) 
\end{align*}
Now, using Lemma~\ref{lem:opt-linear} we get  $\lambda_1^\star = \nf{\gamma}{\|\bmu\|_2\sqrt{1+\gamma^2}}$ and $\lambda_2^\star = \nf{1}{\|\bmu\|_2 \sqrt{1+\gamma^2}}$. Then, we finally apply Lemma~\ref{lem:per-context} to conclude that the accuracy of ERM solution as $p_c \rightarrow 1$ is $0.5 \cdot \erfc(-\rho_1 \sqrt{1+\nf{1}{\gamma}})$ on $c_1$ and $0.5 \cdot \erfc(-\nf{\rho_1(\gamma-1)}{\sqrt{\gamma^2+\nf{1}{\gamma}}})$ on $c_2$.   
\end{proof}

\emph{\ours.~~} Recall that $\rho_2 \eqdef \nf{\|\bmu\|_2}{\sqrt{2}\eta}$. Thus, using arguments similar to the ones in the proof of Lemma~\ref{lem:opt-linear}, the optimal context predictor in  $\mc{W}_1$ will have accuracy of $0.5 \cdot \erfc(-\rho_2)$ on the context prediction problem. Here, we treated the context prediction problem as binary classification with balanced context labels. We can always do this since we have population access to $\mathsf{P}_{c_1}$ and $\mathsf{P}_{c_2}$, and thus we can upsample the examples from the minority context.  
In our simplified setting, the feature predictor is given directly by the context predictor since each context maps to a unique annotation.

Now, to train the target predictor we use ground truth annotations, and given population access we assume that each data point also has the corresponding ground truth annotation, \ie if the datapoint is from context $c_1$, then the annotation is $\rm{C}_1$, else it is $\rm{C}_2$. Consequently, using Lemma~\ref{lem:subspace} and given classification calibrated exponential loss, we can rewrite the optimization problem for \ours as:
\begin{align*}
    \sup\limits_{\substack{\lambda_1 \in [-1, 1], \\ \lambda_2^2 = 1-\lambda_1^2}} \; p_c \cdot \erfc\paren{-\rho_1 \cdot \frac{\lambda_1 + \lambda_2}{\sqrt{\lambda_1^2 + \gamma \lambda_2^2}}} \; + \; (1-p_c) \cdot \erfc\paren{-\rho_1 \cdot \frac{\lambda_1}{\sqrt{\lambda_1^2 + \nf{\lambda_2^2}{\gamma}}}}
\end{align*}

For all $\lambda_2 > 0$, we know that $\nf{\lambda_1 + \lambda_2}{\sqrt{\lambda_1^2 + \gamma \lambda_2^2}} > \nf{\lambda_1}{\sqrt{\lambda_1^2 + (1/\gamma) \cdot \lambda_2^2}}$, when $\gamma \ll 1$. When $p_c \geq 0.5$, then $p_c \cdot \erfc\paren{-\rho_1 \cdot \nicefrac{\lambda_1 + \lambda_2}{\sqrt{\lambda_1^2 + \gamma \lambda_2^2}}} \geq (1-p_c) \cdot  \erfc\paren{-\rho_1 \cdot \nicefrac{\lambda_1}{\sqrt{\lambda_1^2 + \nf{\lambda_2^2}{\gamma}}}}$ for all values of $\lambda_1 \in [0, 1]$. Thus, from Lemma~\ref{lem:opt-linear} we conclude that $\lambda_1^\star = \nf{1}{\|\bmu\|_2\sqrt{1+\gamma^2}}\brck{\bmu\gamma, \bmu, \mathbf{0}_d}$. When we have perfect ground truth annotations, then plugging this value into the equation we above, we find that the accuracy on $c_1$ is $0.5 \cdot \erfc\paren{-\rho_1 \sqrt{1+\nf{1}{\gamma}}}$ and on $c_2$ is $0.5 \cdot \erfc\paren{-\frac{\rho_1}{\sqrt{1+\nf{1}{\gamma^3}}}}$.  

At test time, when we do not have perfect feature annotations on each input, we use the trained feature predictor which has an accuracy of $0.5 \cdot \erfc(-\rho_2)$. Thus the accuracy of \ours on $c_1$ is $\geq (0.5 \cdot \erfc(-\rho_2)) \cdot \paren{0.5 \cdot \erfc\paren{-\rho_1 \sqrt{1+\nf{1}{\gamma}}}} = 0.25 \cdot \erfc(-\rho_2)\cdot\erfc\paren{-\rho_1 \sqrt{1+\nf{1}{\gamma}}}$. Similarly on $c_2$ it is $\geq (0.5 \cdot \erfc(-\rho_2)) \cdot (0.5 \cdot \erfc\paren{-\nicefrac{\rho_1}{\sqrt{1+\nf{1}{\gamma^3}}}}) = 0.25 \cdot \erfc(-\rho_2) \cdot \erfc \paren{- \nicefrac{\rho_1}{\sqrt{1+\nf{1}{\gamma^3}}}}$.

\subsection{Proof for Corollary~\ref{corr-ours}}

\begin{corollary}[Almost Bayes optimality of ENP]
\label{corr-ours-restated}
As non-spurious feature predictor becomes easier to learn $(\eta \rightarrow 0)$, the ratio of accuracies for \ours solution and Bayes optimal predictor approaches $1$ on $c_1$ and $\nf{\erfc(-\nf{\rho_1}{\sqrt{1+\nf{1}{\gamma^3}}})}{\erfc(-\rho_1)}$ on $c_2$. 
\end{corollary}

\begin{proof}
The proof of this corollary directly uses the results regarding the asymptotic performance of \ours from Theorem~\ref{thm:pop-loss-restated}. Since 
$\lim_{\eta\rightarrow 0} \erfc\paren{-\frac{\|\bmu\|_2}{\sqrt{2}\eta}} = 2$, the performance of \ours on $c_1$ approaches $0.5 \cdot \erfc(-\rho_1 \sqrt{1+1/\gamma})$ which is Bayes optimal on $c_1$. Similarly, on $c_2$ it approaches $0.5 \cdot \erfc(- \nf{\rho_1}{\sqrt{1+\nf{1}{\gamma^3}}})$. From this we get the performance ratios stated in Corollary~\ref{corr-ours-restated}. 
\end{proof}



\subsection{Discussion on generalization error for \ours vs. ICC.}
\label{app:gen-error}

In the previous sections, for the class of linear predictors, we say  that the asymptotic error for \ours is lower than IRM and conDRO on context $c_1$ and lower than ERM on $c_2$, under some conditions on problem parameters $\gamma, p_c$. 
Here, we will discuss why \ours performs better than ICC given only finite samples from the distribution. The main intuition behind this is that ICC learns a separate predictor for each context and consequently fails to learn the shared feature $x_1$ jointly using samples from both. Thus, for the minority context the learned predictor would generalize poorly. On the other hand, \ours learns a single predictor for both contexts and instead uses different augmentations for samples from each context. This allows \ours to use samples from the majority context to learn the shared feature $x_1$ that works well on the minority context as well. 

We will now formalize this argument by relying upon 
 existing generalization bounds in prior works  for $l_2$ norm bounded linear predictors. Specifically, we reuse the following generalization bound that is derived using a union bound argument, and thus is applicable to any linear predictor in $\mc{W}_1$ (including ERM estimate). 



\begin{lemma}[Corollary 4 from ~\cite{kakade2008complexity}]
Let $\ell$ be a $L$-Lipschitz loss function, $\mc{S}$ a closed convex set and $1/p + 1/q = 1$. Suppose that $\mc{X} = \{x | \norm{x}_{p} \leq X\}$ and $\mc{W} = \{ w \in \mc{S} | \norm{w}_{q} \leq W \}$.  Then we have for any $\delta > 0$, the generalization error of any $w \in \mc{W}$ is bounded with probability $\geq 1-\delta$.
\begin{equation}
    \ell(\langle w, x\rangle, y) - \frac{1}{n}\sum\limits_{i=1}^{n} \ell(\langle w, x^{(i)}\rangle, y^{(i)}) \leq LXW \sqrt{\frac{p-1}{n}} + LXW \sqrt{\frac{\log(1/\delta)}{2n}}
\end{equation}
\label{lem:kakade-lemma}
\end{lemma}
In particular, when we consider $p=q=2$ and our bounded set of predictors $\mc{W}_{1}$, we recover the bound:
\begin{equation}
    \ell(\langle w, x\rangle, y) - \frac{1}{n}\sum\limits_{i=1}^{n} \ell(\langle w, x^{(i)}\rangle, y^{(i)}) \leq LX \sqrt{\frac{1}{n}} + LX \sqrt{\frac{\log(1/\delta)}{2n}}
\end{equation}


In order to use the above result, we need a high probability bound over the $l_2$ norm of the covariates: $\|x\|_2$ (denoted in the lemma as $X$), which we look into next.




\newcommand{\innerprod}[2]{\langle #1, #2\rangle}
\newcommand{\wtx}{\innerprod{w}{\x}}

\begin{proposition}[high probability bound over \textcolor{black}{$\norm{x}_{2}$}]
\label{prp:hp-norm-bound}
With probability $ \geq 1- \frac{\delta}{2}$, we can bound \textcolor{black}{$\norm{x}_{2}$} using Lemma~\ref{lem:functions-of-gaussian}, 
\begin{align*}
    \norm{x}_{2} \lsim \max\{\sigma/\sqrt{\gamma}, \eta\} \paren{ \sqrt{2 (\log \nicefrac{2}{\delta})} + \sqrt{3d}} + \sqrt{3 \|\bmu\|_2^2}
\end{align*}
\end{proposition}

\begin{proof}
Recall that conditioned on the label and context $\x$ follows a multivariate Gaussian distribution, as specified by \eqref{eq:data-dist}. Now, for a multivariate Gaussian distribution centered at $v \in \Real^{3d}$ and with covariance $\Sigma \in \Real^{3d \times 3d}$, we can use triangle inequality to conclude that $\|\x\|_2 \leq \|v\|_2 + \|\Sigma^{1/2}z\|_2$. This is because we can write $x = v + \Sigma^{1/2}z$ where $z \sim \mc{N}(\mathbf{0}_{3d}, \mathbf{I}_{3d})$

Hence, all we need to do is get a high probability bound over $\|\Sigma^{1/2}z\|_2$ which is a function of $3d$ independent Gaussian variables. Thus, we can apply the concentration bound in Lemma~\ref{lem:functions-of-gaussian}.
But before that, we need to compute the Lipschitz constant for the the function $z \mapsto \|\Sigma^{1/2}z\|_2$ in the euclidean norm.
\begin{align}
    \label{eq:lip-const}
    |\|\Sigma^{1/2}z_1\|_2 - \|\Sigma^{1/2}z_2\|_2| \; \leq \; \|\Sigma^{1/2}(z_1 - z_2)\|_2 \; \leq \; \sqrt{\opnorm{\Sigma}} \cdot \|z_1 - z_2\|_2
\end{align}

Next, with the Lipschitz constant as $\sqrt{\opnorm{\Sigma}}$ we use Lemma~\ref{lem:functions-of-gaussian}, to arrive at the following inequality which holds with probability at least $1-\frac{\delta}{2}$.
\begin{align}
    \label{eq:hp-bd}
    \norm{\x}_{2} \leq \sqrt{2\opnorm{\Sigma} \cdot \log \nf{2}{\delta}} + \E[\|\Sigma^{1/2}z\|_2] + \|v\|_2
\end{align}

Finally, we can use Jensen to bound $\E[\|\Sigma^{1/2}z\|_2]$, \ie 
\begin{align}
    \E\brck{\sqrt{\|\Sigma^{1/2}z\|_2^2}} \leq \sqrt{\E\brck{{\|\Sigma^{1/2}z\|_2^2}}} = \sqrt{\tr\paren{\Sigma}}.
\end{align}
Here, we simplified $\E\brck{\|\Sigma^{1/2}z\|_2^2}$ in the following way: 
\begin{align*}
    \E\brck{\|\Sigma^{1/2}z\|_2^2} = \E \brck{\tr\paren{z^\top \Sigma z}} =  \tr\paren{\Sigma \cdot \E\brck{zz^\top}} = \tr\paren{\Sigma}  
\end{align*}

Since the upper bound worsens with $\|\Sigma\|_\mathrm{op}$ and $\tr(\Sigma)$, we consider the covariance matrix of the Gaussian with the worst $\|\Sigma\|_\mathrm{op}$ and $\tr(\Sigma)$ over the choice of context and label. Recall that $\gamma \ll 1$. Thus, we take $\Sigma$ as determined by context $c_2$, \ie it is given by the following diagonal matrix: $\Sigma = \textrm{diag}(\sigma^2, \sigma^2 \ldots, \sigma^2, \nf{\sigma^2}{\gamma}, \nf{\sigma^2}{\gamma}, \ldots, \nf{\sigma^2}{\gamma}, \eta^2, \eta^2, \ldots, \eta^2)$. Plugging in $\opnorm(\Sigma) \leq \max\{\eta, \nf{\sigma}{\sqrt{\gamma}}\}$ and $\tr(\Sigma) \leq 3d\opnorm(\Sigma)$, and $\|v\|_2 = \sqrt{3\|\bmu\|_2^2}$ into the equation: $ \norm{\x}_{2} \leq \sqrt{2\opnorm{\Sigma} \cdot \log \nf{2}{\delta}} + \sqrt{\tr(\Sigma)} + \|v\|_2$, we get the result in the statement of Proposition~\ref{prp:hp-norm-bound}, \ie, with probability $\geq 1-\nf{\delta}{2}$,
\begin{align*}
\norm{\x}_{2} \leq \sqrt{2 \max \{(\sigma^{2}/\gamma), \eta^{2}\} \cdot \log \nicefrac{2}{\delta}} + \sqrt{3d \max \{(\sigma^{2}/\gamma), \eta^{2}\}} + \sqrt{3\|\bmu\|_2^2}.
\end{align*}
\end{proof}

\begin{lemma}[Lipschitz functions of Gaussians from \cite{wainwright2019high}]
\label{lem:functions-of-gaussian}
Let $X_1, \ldots, X_n$ be a vector of \textit{i.i.d.} Gaussian variables and $f:\Real^n \mapsto \Real$ be $L$-Lipschitz with respect to the Euclidean norm. Then the random variable $f(X) - \E[f(X)]$ is sub-Gaussian with parameter at most $L$, thus:
\begin{align*}
    \Prob[|f(X) - \E[f(X)]| \geq t] \leq 2\cdot \exp{\paren{-\frac{t^2}{2L^2}}}, \;\; \forall \, t\geq 0.
\end{align*}
\end{lemma}



We can now use the high probability bound on $\norm{x}_{2}$ from Proposition~\ref{prp:hp-norm-bound}
in the result in Lemma~\ref{lem:kakade-lemma}. We will use $L$ to denote the Lipschitz constant of the exponential loss. Note that $L$ is finite since we know $\|x\|_2$ is bounded. We also use $n_0$ to denote the number of samples from minority context. 
Finally we apply union bound over the result in Proposition~\ref{prp:hp-norm-bound} and Lemma~\ref{lem:kakade-lemma} to get the following result that bounds the generalization error on the minority context. 

With high probability $1-\delta$, $\forall w \in \mc{W}_1$ we have:
\begin{align*}
     \ell(\langle w, \x \rangle, \y)  - \frac{1}{n} \sum_{i=1}^{n} \ell(\langle w, \x^{(i)} \rangle, \y^{(i)})  \leq  L \paren{\max\{\frac{\sigma}{\sqrt{\gamma}}, \eta\}\paren{\sqrt{2\paren{\log{\frac{2}{\delta}}}}+\sqrt{3d}} + \sqrt{3}\norm{\bmu}_{2}}\paren{\frac{1}{\sqrt{n_{0}}}  +  \sqrt{\frac{\log (2/\delta)}{2n_{0}}}}
\end{align*}

Given this generalization bound we now analyze the generalization gaps for ICC and ENP predictors on the minority context. We will use $c_0$ to denote the constant $\sqrt{3}\|\bmu\|_2 + (\sqrt{3d} + \sqrt{\log(2/\delta)})\max(\sigma/\sqrt{\gamma}, \eta)$.

Recall that ICC simply runs ERM on points coming from each context individually. As a result, we can directly use the above generalization result to bound the generalization gap of ICC on the minority context which has $n_0$ labeled points. In our setting, the context assignment is modeled as a biased coin flip with probability $p_{c}$ for context $c_1$. Thus denoting the number of points in the minority context as $n_{0}$, we have that $n_{0} \sim \text{binom}(n, 1-p_{c})$, were $n$ is the total dataset size. We have that $\mathbb{E}[{n_{0}}] = n(1-p_{c})$ and $\abs{n_{0}-n(1-p_{c})} = \mc{O}_p({1/\sqrt{n}})$ by the Central Limit Theorem. This yields that the generalization bound is $\mc{O}_p(Lc_0(1/ \sqrt{n(1-p_{c})}  +  \sqrt{\log(2/\delta)/n(1-p_{c})}))$, where $c_0$ is as defined above.

In order to analyze ENP, we assume that (a) we have access to the ground-truth feature annotations, and (b) that we observe the samples after spurious features have been masked. 
Effectively, we consider that we are learning a linear predictor over the input space $\rm{C}_{1} \circ {\x}$ when sample is from context $c_1$ and over input space $\rm{C}_2 \circ \x$ when sample is from $c_2$. 
Now, the bound over the constant $X$ is given by a high probability bound over $\rm{C}_1 \circ \x$ and $\rm{C}_2 \circ \x$. 
Trivially, both of these are upper bounded by $\|x\|_2$. 
While this constant remains the same as in ICC, the key difference is that for ICC the bound is realized with only $n_0$ minority samples, but since ENP trains jointly on samples from both datasets the generalization bound is realized by all $n$ samples.
Consequently, given a dataset of $n$ points, we have a generalization bound that is $\mc{O}(Lc_0 (1/ \sqrt{n}  + \sqrt{\log(1/\delta)/n}))$ where $c_0$ is the constant defined above.

We can summarize the above comparison between ICC and ENP on the minority context in terms of the following result on the estimation error of the two estimators.

\begin{theorem}[estimation error]
\label{thm:est-error}
When the exponential loss $\ell$ is optimized over $\mc{W}_1$ using finite samples in $\hat{\mathsf{P}}_n$, then with probability $\geq 1-\delta$ the generalization error on the minority context $\cont_2$  is 
  $\mc{O}_p(Lc_0(\nf{1}{\sqrt{n(1-p_c)}} + \sqrt{\nf{\log(2/\delta)}{n(1-p_c)}}))$
  for the solution found by 
ICC~\eqref{eq:icc-toy}, and  $\mc{O}(Lc_0(\nf{1}{\sqrt{n}} + \sqrt{\nf{\log(2/\delta)}{n}}))$ for the solution found by \ours. Here, $c_0 = \sqrt{3} \|\bmu\|_2 + (\sqrt{3d} + \sqrt{\log(2/\delta)})\max(\nf{\sigma}{\sqrt{\gamma}}, \eta)$. 
\end{theorem}